\documentclass[letterpaper,11pt]{article}

\usepackage{diagbox}
\usepackage{makecell}
\usepackage{booktabs}
\usepackage{tikz}
\usepackage{amsmath}
\usepackage{amsfonts}
\usepackage{amssymb}
\usepackage[square,sort,comma,numbers]{natbib}
\usepackage{amsthm}
\usepackage{url,ifthen}
\usepackage{enumitem}
\usepackage{srcltx}
\usepackage{multirow}
\usepackage{boxedminipage}
\usepackage[margin=1.1in]{geometry}
\usepackage{nicefrac}
\usepackage{xspace}
\usepackage{graphicx}
\usepackage{color}
\usepackage{subfigure}
\usepackage{colortbl}
\usepackage{setspace}

\definecolor{DarkGreen}{rgb}{0.1,0.5,0.1}
\definecolor{DarkRed}{rgb}{0.5,0.1,0.1}
\definecolor{DarkBlue}{rgb}{0.1,0.1,0.5}
\definecolor{Gray}{rgb}{0.2,0.2,0.2}

\newcommand\blfootnote[1]{%
  \begingroup
  \renewcommand\thefootnote{}\footnote{#1}%
  \addtocounter{footnote}{-1}%
  \endgroup
}

\usepackage{listings}
\lstdefinestyle{mystyle}{
    commentstyle=\color{DarkBlue},
    keywordstyle=\color{DarkRed},
    numberstyle=\tiny\color{Gray},
    stringstyle=\color{DarkGreen},
    basicstyle=\footnotesize,
    breakatwhitespace=false,         
    breaklines=true,                 
    captionpos=b,                    
    keepspaces=true,                 
    numbers=left,                    
    numbersep=5pt,                  
    showspaces=false,                
    showstringspaces=false,
    showtabs=false,                  
    tabsize=2
}
\usepackage[small]{caption}
\usepackage{hyperref}
\hypersetup{
    unicode=false,          
    pdftoolbar=true,        
    pdfmenubar=true,        
    pdffitwindow=false,      
    pdfnewwindow=true,      
    colorlinks=true,       
    linkcolor=DarkBlue,          
    citecolor=DarkGreen,        
    filecolor=DarkRed,      
    urlcolor=DarkBlue,          
    pdftitle={},
    pdfauthor={},
}

\def\draft{1}

\def\submit{0}

\ifnum\draft=1 
    \def\ShowAuthNotes{1}
\else
    \def\ShowAuthNotes{0}
\fi

\ifnum\submit=1
\newcommand{\forsubmit}[1]{#1}
\newcommand{\forreals}[1]{}
\else
\newcommand{\forreals}[1]{#1}
\newcommand{\forsubmit}[1]{}
\fi

\ifnum\ShowAuthNotes=1
\newcommand{\authnote}[2]{{ \footnotesize \bf{\color{DarkRed}[#1's Note:
{\color{DarkBlue}#2}]}}}
\else
\newcommand{\authnote}[2]{}
\fi

\newtheorem{theorem}{Theorem}[section]
\newtheorem{remark}[theorem]{Remark}
\newtheorem{lemma}[theorem]{Lemma}
\newtheorem{corollary}[theorem]{Corollary}
\newtheorem{proposition}[theorem]{Proposition}

\theoremstyle{definition}
\newtheorem{definition}[theorem]{Definition}

\newtheorem{example}[theorem]{Example}

\newtheoremstyle{example_contd}
{\topsep} {\topsep}%
{}
{}
{\bfseries}
{.}
{1em}
{\thmname{#1} \thmnumber{ #2}\thmnote{#3} (continued)}

\theoremstyle{example_contd}
\newtheorem*{example_contd}{Example}

\newcommand{\chapterref}[1]{\hyperref[ch:#1]{Chapter~\ref{ch:#1}}}

\newcommand{\claimref}[1]{\hyperref[claim:#1]{Claim~\ref{claim:#1}}}

\newcommand{\corollaryref}[1]{\hyperref[cor:#1]{Corollary~\ref{cor:#1}}}

\newcommand{\definitionref}[1]{\hyperref[def:#1]{Definition~\ref{def:#1}}}

\newcommand{\equationref}[1]{\hyperref[eq:#1]{Equation~\ref{eq:#1}}}

\newcommand{\factref}[1]{\hyperref[fact:#1]{Fact~\ref{fact:#1}}}

\newcommand{\figureref}[1]{\hyperref[fig:#1]{Figure~\ref{fig:#1}}}

\newcommand{\tableref}[1]{\hyperref[tab:#1]{Table~\ref{tab:#1}}}

\newcommand{\itemref}[1]{\hyperref[item:#1]{Item~(\ref{item:#1})}}

\newcommand{\lemmaref}[1]{\hyperref[lem:#1]{Lemma~\ref{lem:#1}}}

\newcommand{\propref}[1]{\hyperref[prop:#1]{Proposition~\ref{prop:#1}}}

\newcommand{\propositionref}[1]{\hyperref[prop:#1]{Proposition~\ref{prop:#1}}}

\newcommand{\remarkref}[1]{\hyperref[rem:#1]{Remark~\ref{rem:#1}}}

\newcommand{\sectionref}[1]{\hyperref[sec:#1]{Section~\ref{sec:#1}}}

\newcommand{\theoremref}[1]{\hyperref[thm:#1]{Theorem~\ref{thm:#1}}}

\usepackage[utf8]{inputenc}
\usepackage[T1]{fontenc}
\usepackage{kpfonts}
\usepackage{microtype}

\newcommand{\Esymb}{\mathbb{E}}

\DeclareMathOperator*{\E}{\Esymb}

\newcommand{\defeq}{\stackrel{\small \mathrm{def}}{=}}
\renewcommand{\leq}{\leqslant}

\renewcommand{\geq}{\geqslant}

\newcommand{\abs}[1]{\lvert#1\rvert}

\newcommand{\norm}[1]{\lVert#1\rVert_2}

\newcommand{\R}{\mathbb{R}}

\renewcommand{\D}{\mathcal D}

\usepackage{bm}

\newcommand{\ignore}[1]{}

\DeclareMathOperator*{\argmin}{arg\,min}
\DeclareMathOperator*{\argmax}{arg\,max}
\newcommand{\reg}{\text{reg}}

\renewcommand{\epsilon}{\varepsilon}

\newcommand{\remove}[1]{}

\renewcommand{\labelitemi}{--}

\newenvironment{enum}
{\begin{enumerate}[noitemsep,topsep=0pt,parsep=0pt,partopsep=0pt]}
{\end{enumerate}}

\newcommand{\lgrad}[2]{\E_{Z \sim \D(#1)}\nabla_\theta \ell(Z; #2)}
\newcommand{\ploss}[2]{\E_{Z \sim \D(#1)}\ell(Z; #2)}
\newcommand{\plossemp}[3]{\E_{Z \sim \D^{#2}(#1)}\ell(Z; #3)}

\newcommand{\elgrad}[2]{\E_{Z \sim {#1}}\nabla_\theta \ell(Z; #2)}

\newcommand{\gd}{G_{\text{gd}}}
\newcommand{\pregloss}[2]{\E_{Z \sim \D({#1})}\ell^{\reg}(Z; #2)}

\newcommand{\PR}{\mathrm{PR}}
\newcommand{\thetaSL}{{\theta_{\mathrm{SL}}}}
\newcommand{\thetaPO}{{\theta_{\mathrm{PO}}}}
\newcommand{\thetaPS}{{\theta_{\mathrm{PS}}}}
\newcommand{\thetaSE}{{\theta_{\mathrm{SE}}}}
\newcommand{\map}{{\mathrm{DPR}}}

\title{Performative Prediction}

\author{\\Juan C. Perdomo*~~~~Tijana Zrnic*~~~Celestine Mendler-D\"unner~~~~Moritz Hardt\\
{\small \{jcperdomo, tijana.zrnic, mendler, hardt\}@berkeley.edu}
\\ \\University of California, Berkeley}

\begin{document}

\maketitle


\begin{abstract}
When predictions support decisions they may influence the outcome they aim to predict. We call such predictions \emph{performative}; the prediction influences the target. Performativity is a well-studied phenomenon in policy-making that has so far been neglected in supervised learning. When ignored, performativity surfaces as undesirable distribution shift, routinely addressed with retraining.

We develop a risk minimization framework for performative prediction bringing together concepts from statistics, game theory, and causality. A conceptual novelty is an equilibrium notion we call performative stability. Performative stability implies that the predictions are calibrated not against past outcomes, but against the future outcomes that manifest from acting on the prediction. Our main results are necessary and sufficient conditions for the convergence of retraining to a performatively stable point of nearly minimal loss.

In full generality, performative prediction strictly subsumes the setting known as \emph{strategic classification}. We thus also give the first sufficient conditions for retraining to overcome strategic feedback effects.
\end{abstract}

\section{Introduction}

\blfootnote{* Equal contribution.}Supervised learning excels at pattern recognition. When used to support consequential decisions, however, predictive models can trigger actions that influence the outcome they aim to predict. We call such predictions \emph{performative}; the prediction causes a change in the distribution of the target variable.

Consider a simplified example of predicting credit default risk. A bank might estimate that a loan applicant has an elevated risk of default, and will act on it by assigning a high interest rate. In a self-fulfilling prophecy, the high interest rate further increases the customer\textquotesingle s default risk. Put differently, the bank\textquotesingle s predictive model is not calibrated to the outcomes that manifest from acting on the model.

Once recognized, performativity turns out to be ubiquitous. Traffic predictions influence traffic patterns, crime location prediction influences police allocations that may deter crime, recommendations shape preferences and thus consumption, stock price prediction determines trading activity and hence prices.

When ignored, performativity can surface as a form of \emph{distribution shift}. As the decision-maker acts according to a predictive model, the distribution over data points appears to change over time. In practice, the response to such distribution shifts is to frequently \emph{retrain} the predictive model as more data becomes available. Retraining is often considered an undesired~--- yet necessary --- cat and mouse game of chasing a moving target. 

What would be desirable from the perspective of the decision maker is a certain equilibrium where the model is optimal for the distribution it induces. Such equilibria coincide with the stable points of retraining, that is, models invariant under retraining. Performativity therefore suggests a different perspective on retraining, exposing it as a natural equilibrating dynamic rather than a nuisance.

This raises fundamental questions. When do such stable points exist? How can we efficiently find them? Under what conditions does retraining converge? When do stable points also have good predictive performance? In this work, we formalize performative prediction, tying together conceptual elements from statistical decision theory, causal reasoning, and game theory. We then resolve some of the fundamental questions that performativity raises.
 
\subsection{Our contributions}

We put performativity at the center of a decision-theoretic framework that extends the classical statistical theory underlying risk minimization. The goal of risk minimization is to find a decision rule, specified by model parameters~$\theta$, that performs well on a fixed joint distribution~$\mathcal{D}$ over covariates~$X$ and an outcome variable~$Y$.

Whenever predictions are performative, the choice of predictive model affects the observed distribution over instances $Z = (X, Y)$. We formalize this intuitive notion by introducing a map $\mathcal{D}(\cdot)$ from the set of model parameters to the space of distributions. For a given choice of parameters~$\theta$, we think of $\D(\theta)$ as the distribution over features and outcomes that results from making decisions according to the model specified by $\theta$. This mapping from predictive model to distribution is the key conceptual device of our framework.

A natural objective in performative prediction is to evaluate model parameters~$\theta$ on the resulting distribution~$\D(\theta)$ as measured via a loss function~$\ell$. This results in the notion we call \emph{performative risk}, defined as
\begin{equation*}
\PR(\theta) \defeq \E_{Z\sim \D(\theta)} \ell(Z; \theta)\,.
\end{equation*}

The difficulty in minimizing $\PR(\theta)$ is that the distribution itself depends on the argument~$\theta$, a dependence that defeats traditional theory for risk minimization. Moreover, we generally envision that the map~$\D(\cdot)$ is unknown to the decision maker.

Perhaps the most natural algorithmic heuristic in this situation is a kind of fixed point iteration: repeatedly find a model that minimizes risk on the distribution resulting from the previous model, corresponding to the update rule
\begin{equation*}
\theta_{t+1} = \arg\min_{\theta} \E_{Z\sim \D(\theta_t)} \ell(Z; \theta)\,.
\end{equation*}
We call this procedure \emph{repeated risk minimization}. We also analyze its empirical counterpart that works in finite samples. These procedures exemplify a family of \emph{retraining} heuristics that are ubiquitous in practice for dealing with all kinds of distributions shifts irrespective of cause.

When repeated risk minimization converges in objective value the model has minimal loss on the distribution it entails:
\begin{equation*}
\PR(\theta) = \min_{\theta'} \E_{Z\sim \D(\theta)} \ell(Z; \theta')\,.
\end{equation*}
We refer to this condition as \emph{performative stability}, noting that it is neither implied by nor does it imply minimal performative risk.

Our central result can be summarized informally as follows.

\begin{theorem}[Informal]
If the loss is smooth, strongly convex, and the mapping~$\D(\cdot)$ is sufficiently Lipschitz, then
repeated risk minimization converges to performative stability at a linear rate.

Moreover, if any one of these assumptions does not hold, repeated risk minimization can fail to converge at all.
\end{theorem}

The notion of Lipschitz continuity here refers to the Euclidean distance on model parameters and the Wasserstein distance on distributions. Informally, it requires that a small change in model parameters $\theta$ does not have an outsized effect on the induced distribution~$\D(\theta)$.

In contrast to standard supervised learning, convexity alone is \emph{not} sufficient for convergence in objective value, even if the other assumptions hold. Performative prediction therefore gives a new and interesting perspective on the importance of strong convexity. 

Strong convexity has a second benefit. Not only does retraining converge to a stable point at a linear rate, this stable point also approximately minimizes the performative risk.

\begin{theorem}[Informal]
If the loss is Lipschitz and strongly convex, and the map~$\D(\cdot)$ is Lipschitz, all stable points and performative optima lie in a small neighborhood around each other.
\end{theorem}

Recall that performative stability on its own does not imply minimal performative risk. What the previous theorem shows, however, is that strong convexity guarantees that we can approximately satisfy both.

We complement our main results with a case study in \emph{strategic classification}. Strategic classification aims to anticipate a strategic response to a classifier from an individual, who can change their features prior to being classified. We observe that strategic classification is a special case of performative prediction. On the one hand, this allows us to transfer our technical results to this established setting. In particular, our results are the first to give a guarantee on repeated risk minimization in the strategic setting. On the other hand, strategic classification provides us with one concrete setting for what the mapping~$\D(\cdot)$ can be. We use this as a basis of an empirical evaluation in a semi-synthetic setting, where the initial distribution is based on a real data set, but the distribution map is modeled.

\subsection{Related work}

Performativity is a broad concept in the social sciences, philosophy, and economics~\cite{mackenzie2007economists,healy2015performativity}.  Below we focus on the relationship of our work to the most relevant technical scholarship.

\paragraph{Learning on non-stationary distributions.} A closely related line of work considers the problem of \emph{concept drift}, broadly defined as the problem of learning when the target distribution over instances drifts with time. This setting has attracted attention both in the learning theory community \cite{kuh1991learning, bartlett_concepts, bartlett_slowlychanging} and by machine learning practitioners \cite{gama2014survey}. 

Concept drift is more general phenomenon than performativity in that it considers arbitrary sources of shift. However, studying the problem at this level of generality has led to a number of difficulties in creating a unified language and objective \cite{gama2014survey, webb2016characterizing}, an issue we circumvent by assuming that the population distribution is determined by the deployed predictive model.  Importantly, this line of work also discusses the importance of retraining \cite{vzliobaite2010learning,gama2014survey}. However, it stops short of discussing the need for stability or analyzing the long-term behavior of retraining.

\paragraph{Strategic classification.} Strategic classification recognizes that individuals often adapt to the specifics of a decision rule so as to gain an advantage (see, e.g., ~\cite{dalvi2004adversarial, bruckner2012static, hardt2016strategic, khajehnejad2019optimal}). Recent work in this area considers issues of incentive design~\cite{kleinberg2019classifiers, miller2019strategic, shavit2020learning,bechavod2020causal}, control over an algorithm~\cite{burrell2019when}, and fairness concerns~\cite{hu2019disparate, milli2019social}. Importantly, the concurrent work of Bechavod et al.~\cite{bechavod2020causal} analyzes the implications of retraining in the context of causal discovery in linear models. Our model of performative prediction includes all notions of strategic adaption that we are aware of as a special case. Unlike many works in this area, our results do not depend on a specific \emph{cost function} for changing individual features. Rather, we rely on an assumption about the sensitivity of the data-generating distribution to changes in the model parameters.

Recently, there has been increased interest within the algorithmic fairness community in classification dynamics. See, for example, Liu et al.~\cite{liu2018delayed}, Hu and Chen~\cite{hu2018short}, and Hashimoto et al.~\cite{hashimoto2018fairness}. The latter work considers repeated risk minimization, but from the perspective of what it does to a measure of disparity between groups.

\paragraph{Causal inference.} The reader familiar with causality can think of $\D(\theta)$ as the interventional distribution over instances $Z$ resulting from a do-intervention that sets the model parameters to $\theta$ in some underlying causal graph. Importantly, this mapping~$\D(\cdot)$ remains fixed and does not change over time or by intervention: deploying the same model at two different points in time must induce the same distribution over observations $Z$. While causal inference focuses on estimating properties of interventional distributions such as treatment effects \cite{pearl2009causality, imbens2015causal}, our focus is on a new stability notion and iterative retraining procedures for finding stable points.

\paragraph{Convex optimization.} The two solution concepts we introduce generalize the usual notion of optimality in (empirical) risk minimization to our new framework of performativity. Similarly, we extend the classical property of gradient descent acting as a contraction under smooth and strongly convex losses to account for distribution shifts due to performativity. Finally, we discuss how different regularity assumptions on the loss function affect convergence of retraining schemes, much like optimization works discuss these assumptions in the context of convergence of iterative optimization algorithms.

\paragraph{Reinforcement learning.} In general, any instance of performative prediction can be reframed as a reinforcement learning or contextual bandit problem. Yet, by studying performative prediction problems within such a broad framework, we lose many of the intricacies of performativity which make the problem interesting and tractable to analyze. We return to discuss some of the connections between both frameworks later on.


\section{Framework and main definitions}

In this section, we formally introduce the principal solution concepts of our framework: performative optimality and performative stability. 

Throughout our presentation, we focus on predictive models that are parametrized by a  vector  $\theta \in \Theta$, where the parameter space $\Theta\subseteq \R^d$ is a closed, convex set. We use capital letters to denote random variables and their lowercase counterparts to denote realizations of these variables. We consider instances $z = (x,y)$ defined as feature, outcome pairs, where $x\in\R^{m-1}$ and $y\in\R$. Whenever we define a variable $\theta^* = \argmin_\theta g(\theta)$ as the minimizer of a function $g$, we resolve the issue of the minimizer not being unique by setting $\theta^*$ to an arbitrary point in the $\argmin_\theta g(\theta)$ set.   

\subsection{Performative optimality}

In supervised learning, the goal is to learn a predictive model $f_\theta$ which minimizes the expected loss with respect to feature, outcome pairs $(x,y)$ drawn i.i.d. from a fixed distribution $\D$. The optimal model $f_{\thetaSL}$ solves the following optimization problem,
\begin{equation*}
\thetaSL = \argmin_{\theta \in \Theta} \E_{Z \sim \D}\ell (Z;\theta),
\end{equation*}
where $\ell(z;\theta)$ denotes the loss of $f_{\theta}$ at a point $z$.

We contrast this with the \emph{performative optimum}. As introduced previously, in settings where predictions support decisions, the manifested distribution over features and outcomes is in part determined by the deployed model. Instead of considering a fixed distribution $\D$, each model $f_\theta$ induces a potentially different distribution $\D(\theta)$ over instances $z$. A predictive model must therefore be evaluated with regard to the expected loss over the distribution $\D(\theta)$ it induces: its \emph{performative risk}. 

\begin{definition}[performative optimality and risk]
\label{def:performative_optimum}
A model $f_{\thetaPO}$ is \emph{performatively optimal} if the following relationship holds:
\[\thetaPO = \argmin_\theta \ploss{\theta}{\theta}.\]
We define $\PR(\theta) \defeq \ploss{\theta}{\theta}$ as the \emph{performative risk}; then, $\thetaPO = \argmin_\theta \PR(\theta)$.
\end{definition}

The following example illustrates the differences between the traditional notion of optimality in supervised learning and performative optima.

\begin{example}[biased coin flip]
\label{example:biasedcointoss}
Consider the task of predicting the outcome of a biased coin flip where the bias of the coin depends on a feature $X$ and the assigned score $f_\theta(X)$. 

In particular, define $\D(\theta)$ in the following way. $X$ is a 1-dimensional feature supported on $\{\pm 1\}$ and $Y~|~X  \sim \text{Bernoulli}(\frac 1 2 + \mu X + \epsilon \theta X)$ with $\mu\in(0,\frac{1}{2})$ and $\epsilon< \frac{1}{2} - \mu$. Assume that the class of predictors consists of linear models of the form $f_\theta(x) = \theta x + \frac 1 2$ and that the objective is to minimize the squared loss: $\ell(z;\theta) = (y - f_\theta(x))^2$.

The parameter $\epsilon$ represents the performative aspect of the model. If $\epsilon=0$, outcomes are independent of the assigned scores and the problem reduces to a standard supervised learning task where the optimal predictive model is the conditional expectation $f_\thetaSL(x) = \E [Y \mid X=x] = \frac 1 2 + \mu x$, with $\thetaSL=\mu$. 

In the performative setting with $\epsilon \neq 0$, the optimal model $\thetaPO$ balances between its predictive accuracy as well as the bias induced by the prediction itself. In particular, a direct calculation demonstrates that 
\begin{equation*}
	\thetaPO = \argmin_{\theta \in [0,1]} \E_{Z \sim \D(\theta)}\left( Y - \theta X - \frac{1}{2} \right)^2 \quad \iff \quad\thetaPO = \frac{\mu}{1 - 2\epsilon}.
\end{equation*} 
Hence, the performative optimum and the supervised learning solution are equal if $\epsilon=0$ and diverge as the performativity strength $\epsilon$ increases.
\end{example}

\subsection{Performative stability}

A natural, desirable property of a model $f_\theta$ is that, given that we use the predictions of $f_\theta$ as a basis for decisions, those predictions are also simultaneously optimal for distribution that the model induces. We introduce the notion of  \emph{performative stability} to refer to predictive models that satisfy this property. 

\sloppy

\begin{definition}[performative stability and decoupled risk]
A model $f_{\thetaPS}$ is \emph{performatively stable} if the following relationship holds:
\begin{equation*}
\thetaPS = \argmin_{\theta} \ploss{\thetaPS}{\theta}.
\end{equation*}
We define $\map(\theta,\theta') \defeq \ploss{\theta}{\theta'}$ as the \emph{decoupled performative risk}; then, $\thetaPS = \argmin_\theta \map(\thetaPS, \theta)$.
\end{definition}
\fussy

A performatively stable model $f_\thetaPS$ minimizes the expected loss on the distribution $\D(\thetaPS)$ resulting from deploying $f_\thetaPS$ in the first place. Therefore, a model that is performatively stable eliminates the need for retraining after deployment since any retraining procedure would simply return the same model parameters. Performatively stable models are \emph{fixed points} of risk minimization. We further develop this idea in the next section. 
 
Observe that performative optimality and performative stability are in general two distinct solution concepts. Performatively optimal models need not be performatively stable and performatively stable models need not be performatively optimal. We illustrate this point in the context of our previous biased coin toss example.
 
\begin{example_contd}[\ref{example:biasedcointoss}]
Consider again our model of a biased coin toss. In order for a predictive model $f_\theta$ to be performatively stable, it must satisfy the following relationship: 
\begin{equation*}
\thetaPS = \argmin_{\theta\in[0,1]} \E_{Z \sim \D(\thetaPS)}\left(Y - \theta X - \frac{1}{2}\right)^2 \quad \iff \quad \thetaPS = \frac{\mu}{1 - \epsilon}.
\end{equation*}
Solving for $\thetaPS$ directly, we see that there is a unique performatively stable point.

Therefore, performative stability and performative optimality need not identify. In fact, in this example they identify if and only if $\epsilon = 0$. Note that, in general, if the map $\D(\theta)$ is constant across $\theta$, performative optima must coincide with performatively stable solutions. Furthermore, both coincide with "static" supervised learning solutions as well.
\end{example_contd}

For ease of presentation, we refer to a choice of parameters $\theta$ as performatively stable (optimal) if the model parametrized by $\theta$, $f_\theta$ is performatively stable (optimal). We will occasionally also refer to performative stability as simply stability.

\begin{remark} Notice that both performative stability and optimality can be expressed via the decoupled performative risk as follows:
\begin{align*}
\quad\quad\quad&\thetaPS \text { is performatively stable }& \Leftrightarrow&&\thetaPS  = \argmin_{\theta}\;\map(\thetaPS, \theta),&\quad\quad\quad\\
&\thetaPO \text { is performatively optimal } &\Leftrightarrow &&\thetaPO = \argmin_{\theta}\;\map(\theta, \theta).&
\end{align*}
\end{remark}


\section{When retraining converges to stable points}

Having introduced our framework for performative prediction, we now address some of the basic questions that arise in this setting and examine the behavior of common machine learning practices, such as retraining, through the lens of performativity.

As discussed previously, performatively stable models have the favorable property that they achieve minimal risk for the distribution they induce and hence eliminate the need for retraining. However, it is a priori not clear that such stable points exist; and even if they do exist, whether we can find them efficiently. Furthermore, seeing as how performative optimality and stability are in general distinct solution concepts, under what conditions can we find models that approximately satisfy both?

In this work, we begin to answer these questions by analyzing two different optimization strategies. The first is retraining, formally referred to as \emph{repeated risk minimization} (RRM), where the exact minimizer is repeatedly computed on the distribution induced by the previous model parameters. The second is \emph{repeated gradient descent} (RGD), in which the model parameters are incrementally updated using a single gradient descent step on the objective defined by the previous iterate. We introduce RGD as a computationally efficient approximation of RRM, which, as we show, adopts many favorable properties of RRM.

Our algorithmic analysis of these methods reveals the existence of stable points under the assumption that the distribution map $\D(\cdot)$ is sufficiently Lipschitz. We  identify necessary and sufficient conditions for convergence to a performatively stable point and establish properties of the objective under which stable points and performative optima are close.

We begin by analyzing the behavior of these procedures when they operate at a population level and then extend our analysis to finite samples.

\subsection{Assumptions}

It is easy to see that one cannot make any guarantees on the convergence of retraining or the existence of stable points without making some regularity assumptions on $\D(\cdot)$. One reasonable way to quantify the regularity of $\D(\cdot)$ is to assume Lipschitz continuity; the Lipschitz constant determines how sensitive the induced distribution is to a change in model parameters. Intuitively, such an assumption captures the idea that, if decisions are made according to similar predictive models, then the resulting distributions over instances should also be similar. We now introduce this key assumption of our work, which we call $\epsilon$\emph{-sensitivity}.

\begin{definition}[$\epsilon$-sensitivity]
\label{def:eps}
We say that a distribution map $\D(\cdot)$ is \emph{$\epsilon$-sensitive} if for all $\theta, \theta' \in \Theta$:
\begin{equation*}
W_1\big(\D(\theta), \D(\theta')\big) \leq \epsilon\|\theta -\theta'\|_2,
\end{equation*}
where $W_1$ denotes the Wasserstein-1 distance, or earth mover's distance.
\end{definition}

The earth mover's distance is a natural notion of distance between probability distributions that provides access to a rich technical repertoire \cite{villani, villani2008}. Furthermore, we can verify that it is satisfied in various settings.

\begin{remark}
A simple example where this assumption is satisfied is for a Gaussian family. Given $\theta = (\mu,\sigma_1,\dots,\sigma_p)\in\R^{2p}$, define $\D(\theta) = \mathcal{N}(\epsilon_1 \, \mu, \epsilon_2^2 \, \text{diag}(\sigma_1^2,\dots,\sigma_p^2))$ where $\epsilon_1,\epsilon_2\in\R$. Then $\D(\cdot)$ is $\epsilon$-sensitive for $\epsilon = \max\big\{|\epsilon_1|,|\epsilon_2|\big\}$.
\end{remark}

In addition to this assumption on the distribution map, we will often make standard assumptions on the loss function $\ell(z;\theta)$ which hold for broad classes of losses. To simplify our presentation, let $\mathcal{Z} \defeq \cup_{\theta\in\Theta} \text{supp}(\D(\theta))$.

\renewcommand{\labelitemi}{$\bullet$}
\begin{itemize}
\item	(\emph{joint smoothness}) We say that a loss function $\ell(z;\theta)$ is $\beta$-jointly smooth if the gradient $\nabla_\theta \ell(z;\theta)$ is $\beta$-Lipschitz in $\theta$ \emph{and} $z$, that is
\begin{equation}
\left\|\nabla_\theta\ell(z;\theta) - \nabla_\theta \ell(z;\theta')\right\|_2 \leq \beta \left\|\theta - \theta'\right\|_2, ~~ \left\|\nabla_\theta\ell(z;\theta) - \nabla_\theta \ell(z';\theta)\right\|_2 \leq \beta \left\|z - z'\right\|_2,
\tag{A1}
\label{ass:a2}
\end{equation}
for all $\theta,\theta'\in\Theta$ and $z,z' \in \mathcal{Z}$.
\item(\emph{strong convexity}) We say that a loss function $\ell(z;\theta)$ is $\gamma$-strongly convex if
\begin{equation}
\ell(z;\theta)\geq \ell(z;\theta') + \nabla_\theta \ell(z;\theta')^\top (\theta-\theta') + \frac{\gamma}{2}\left\|\theta-\theta'\right\|_2^2,
\tag{A2}
\label{ass:a3}
\end{equation}
for all $\theta,\theta'\in\Theta$ and $z\in \mathcal{Z}$. If $\gamma=0$, this assumption is equivalent to convexity.
\end{itemize}

We will sometimes refer to $\frac{\beta}{\gamma}$, where $\beta$ is as in \eqref{ass:a2} and $\gamma$ as in \eqref{ass:a3}, as the condition number.

\subsection{Repeated risk minimization}

We now formally define repeated risk minimization and prove one of our main results: sufficient and necessary conditions for retraining to converge to a performatively stable point.

\begin{definition}[RRM]
\label{def:rrm}
\emph{Repeated risk minimization} (RRM) refers to the procedure where, starting from an initial model $f_{\theta_0}$, we perform the following sequence of updates for every $t\geq 0$:
\begin{align*}
\theta_{t+1} &= G(\theta_{t}) \defeq \argmin_{\theta \in \Theta} \ploss{\theta_{t}}{\theta}.
\end{align*}
\end{definition}

Using a toy example, we again argue that restrictions on the map $\D(\cdot)$ are necessary to enable interesting analyses of RRM, otherwise it might be computationally infeasible to find performative optima, and performatively stable points might not even exist.
\begin{example}
Consider optimizing the squared loss $\ell(z;\theta) = (y-\theta)^2$, where $\theta\in[0,1]$ and the distribution of the outcome $Y$, according to $\D(\theta)$, is a point mass at 0 if $\theta\geq \frac{1}{2}$, and a point mass at 1 if $\theta<\frac{1}{2}$. Clearly there is no performatively stable point, and RRM will simply result in the alternating sequence $1,0,1,0,\dots$. The performative optimum in this case is $\thetaPO=\frac{1}{2}$.
\end{example}

To show convergence of retraining schemes, it is hence necessary to make a regularity assumption on $\D(\cdot)$, such as $\epsilon$-sensitivity.  We are now ready to state our main result regarding the convergence of repeated risk minimization.

\begin{theorem}
\label{theorem:exact_min_strongly_convex}
Suppose that the loss $\ell(z;\theta)$ is $\beta$-jointly smooth \eqref{ass:a2} and $\gamma$-strongly convex \eqref{ass:a3}. If the distribution map $\D(\cdot)$ is $\epsilon$-sensitive, then the following statements are true:
\begin{itemize}
\item[(a)] $\|G(\theta) - G(\theta')\|_2 \leq \epsilon \frac{\beta}{\gamma} \|\theta-\theta'\|_2, ~~ \text{for all } \theta,\theta'\in\Theta.$
\item[(b)] If $\epsilon<\frac{\gamma}{\beta}$,  the iterates $\theta_t$ of RRM converge to a unique performatively stable point $\thetaPS$ at a linear rate:
$\|\theta_t - \thetaPS\|_2\leq \delta \text{ for } t\geq \left(1 - \epsilon \frac{\beta}{\gamma}\right)^{-1}\log\left(\frac{\|\theta_0 - \thetaPS\|_2}{\delta}\right)$.
\end{itemize}
	
\end{theorem}

The main message of this theorem is that in performative prediction, if the loss function is sufficiently "nice" and the distribution map is sufficiently (in)sensitive, then one need only retrain a model a small number of times before it converges to a \emph{unique} stable point. The complete proof of  Theorem~\ref{theorem:exact_min_strongly_convex} can be found in Appendix~\ref{app:proofexact}. Here, we provide the main intuition through a proof sketch.

\begin{proof}[Proof Sketch]

Fix $\theta,\theta'\in\Theta$. Let $f(\varphi)=  \ploss{\theta}{\varphi}$ and $f'(\varphi)=  \ploss{\theta'}{\varphi}$. By applying standard properties of strong convexity and the fact that  $G(\theta)$  is the unique minimizer of $f(\varphi)$, we can derive that
\begin{equation*}
-\gamma \|G(\theta)-G(\theta')\|_2^2 \geq \left(G(\theta)-G(\theta')\right)^\top\nabla f(G(\theta')).
\end{equation*}

Next, we observe that $(G(\theta) - G({\theta'}))^\top\nabla_\theta \ell(z; G(\theta'))$ is $\|G(\theta) - G({\theta'})\|_2 \beta$-Lipschitz in $z$. This follows from applying the Cauchy-Schwarz inequality and the fact that the loss is $\beta$-jointly smooth. Using the dual formulation of the earth mover's distance (Lemma~\ref{lemma:duality}) and $\epsilon$-sensitivity of $\D(\cdot)$, as well as the first-order conditions of optimality for convex functions, a short calculation reveals that
\[(G(\theta) - G({\theta'}))^\top  \nabla f(G(\theta')) \geq - \epsilon\beta\|G(\theta)-G({\theta'})\|_2\|\theta-\theta'\|_2.
\]Claim (a) then follows by combining the previous two inequalities and rearranging. Intuitively, strong convexity forces the iterates to contract after retraining, yet this contraction is offset by the distribution shift induced by changing the underlying model parameters. Joint smoothness and $\epsilon$-sensitivity ensure that this shift is not too large. Part (b) is essentially a consequence of applying the Banach fixed-point theorem to the result of part (a).
\end{proof}

One intriguing insight from our analysis is that this convergence result is in fact tight; removing any single assumption required for convergence by Theorem~\ref{theorem:exact_min_strongly_convex} is enough to construct a counterexample for which RRM diverges.
\begin{proposition}
\label{prop:tightness}
Suppose that the distribution map $\D(\cdot)$ is $\epsilon$-sensitive with $\epsilon>0$. RRM can fail to converge at all in any of the following cases, for any choice of parameters $\beta,\gamma>0$:
\begin{itemize}
	\item[(a)] The loss is $\beta$-jointly smooth and convex, but not strongly convex.
	\item[(b)] The loss is $\gamma$-strongly convex, but not jointly smooth.
	\item[(c)] The loss is $\beta$-jointly smooth and $\gamma$-strongly convex, but $\epsilon \geq\frac{\gamma}{\beta}$.
\end{itemize}
\end{proposition}

We include a counterexample for statement (a), and defer the proofs of (b) and (c) to Appendix \ref{app:prooftightness}.

\begin{proof}[Proof of Proposition~\ref{prop:tightness}(a):]
 Consider the linear loss defined as $\ell\left((x,y);\theta\right) = \beta y \theta$, for $\theta\in[-1, 1]$. Note that this objective is $\beta$-jointly smooth and convex, but not strongly convex. Let the distribution of $Y$ according to $\D(\theta)$ be a point mass at $\epsilon \theta$, and let the distribution of $X$ be invariant with respect to $\theta$. Clearly, this distribution is $\epsilon$-sensitive.

Here, the decoupled performative risk has the following form $\map(\theta,\varphi) = \epsilon \beta \theta \varphi$. The unique performatively stable point is 0. However, if we initialize RRM at any point other than 0, the procedure generates the sequence of iterates $\dots,1,-1,1,-1\dots$, thus failing to converge. Furthermore, this behavior holds for all $\epsilon, \beta > 0$. 
\end{proof}

Proposition~\ref{prop:tightness} suggests a fundamental difference between strong and weak convexity in our framing of performative prediction (weak meaning $\gamma=0$). In supervised learning,  using strongly convex losses generally guarantees a faster rate of optimization, yet asymptotically, the solution achieved with either strongly or weakly convex losses is globally optimal. However, in our framework, strong convexity is in fact \emph{necessary} to guarantee convergence of repeated risk minimization, even for arbitrarily smooth losses and an arbitrarily small sensitivity parameter.

\subsection{Repeated gradient descent}

Theorem~\ref{theorem:exact_min_strongly_convex} demonstrates that repeated risk minimization converges to a unique performatively stable point if the sensitivity parameter $\epsilon$ is small enough. However, implementing RRM requires access to an exact optimization oracle. We now relax this requirement and demonstrate how a simple gradient descent algorithm also converges to a unique stable point.

\begin{definition}[RGD]
\label{def:rgd}
\emph{Repeated gradient descent} (RGD) is the procedure where, starting from an initial model $f_{\theta_0}$, we perform the following sequence of updates for every $t\geq 0$:
$$\theta_{t+1} = \gd(\theta_{t}) \defeq \Pi_{\Theta}\left(\theta_{t} - \eta \lgrad{\theta_{t}}{\theta_{t}}\right),$$
where $\eta> 0$ is a fixed step size and $\Pi_{\Theta}$ denotes the Euclidean projection operator onto $\Theta$.	
\end{definition}

Note that repeated gradient descent only requires the loss $\ell$ to be differentiable with respect to $\theta$. It does not require taking gradients  of the performative risk. Like RRM, we can show that RGD is a contractive mapping for small enough sensitivity parameter $\epsilon$.

\begin{theorem}
\label{theorem:contraction}
Suppose that the loss $\ell(z;\theta)$ is $\beta$-jointly smooth \eqref{ass:a2} and $\gamma$-strongly convex \eqref{ass:a3}. If the distribution map $\D(\cdot)$ is $\epsilon$-sensitive with $\epsilon < \frac{\gamma}{(\beta + \gamma)(1 + 1.5\eta \beta)}$, then RGD with step size $\eta\leq\frac{2}{\beta+\gamma}$ satisfies the following:
\begin{itemize}
	\item[(a)] $\|\gd(\theta)- \gd(\theta')\|_2 \leq \left(1 - \eta\left( \frac{\beta \gamma}{\beta + \gamma}  - \epsilon (1.5\eta\beta^2 + \beta)\right)\right) \norm{\theta - \theta'} < \norm{\theta - \theta'}.$
	\item[(b)] The iterates $\theta_t$ of RGD converge to a unique performatively stable point $\thetaPS$ at a linear rate, $\|\theta_t - \thetaPS\|_2 \leq \delta$ for $t\geq \frac 1 \eta {\left(\frac{\beta\gamma}{\beta + \gamma} - \epsilon (1.5\eta\beta^2 + \beta)\right)^{-1}}\log\left(\frac{\|\theta_0 - \thetaPS\|_2}{\delta}\right)$.
	\end{itemize}
\end{theorem}

The conclusion of Theorem~\ref{theorem:contraction} is a strict generalization of  a classical optimization result which considers a static objective, in which case the rate of contraction is $\left(1 - \eta \frac{\beta\gamma}{\beta + \gamma}\right)$ (see for example Theorem 2.1.15 in \cite{nesterov2013introductory} or Lemma 3.7 in \cite{hardt2016train}). Our rate exactly matches this classical result in the case that $\epsilon = 0$. The proof of Theorem~\ref{theorem:contraction} can be found in  Appendix~\ref{app:proofGD}.

\subsection{Finite-sample analysis}

We now extend our main results regarding the convergence of RRM and RGD to the finite-sample regime. To do so, we leverage the fact that, under mild regularity conditions, the empirical distribution $\D^n$ given by $n$ samples drawn i.i.d. from a true distribution $\D$ is with high probability close to $\D$ in earth mover's distance \cite{Fournier2015}. 

We begin by defining the finite-sample version of these procedures. 
\begin{definition}[RERM \& REGD]
\label{def:rerm}
	Define \emph{repeated empirical risk minimization} (RERM) to be the procedure where starting from a model $f_{\theta_{0}}$ at every iteration $t \geq 0$, we collect $n_t$ samples from $\D(\theta_t)$ and perform the update:
\begin{equation*}
\theta_{t+1} = G^{n_t}\left(\theta_{t}\right) \;\defeq\; \argmin_\theta \plossemp{\theta_{t}}{n_t}{\theta}.
\end{equation*}
Similarly, define  \emph{repeated empirical gradient descent} (REGD) to be the optimization procedure with update rule:
\begin{equation*}
\theta_{t+1} = \gd^{n_t}(\theta_t) \defeq \Pi_\Theta \left(\theta_t - \eta \E_{Z\sim \D^{n_{t}}(\theta_{t})} \nabla_\theta \ell(Z;\theta_{t})\right).
\end{equation*}
Here, $\eta>0$ is a step size and $\Pi_{\Theta}$ denotes the Euclidean projection operator onto $\Theta$.

\end{definition}

The following theorem illustrates that with enough samples collected at every iteration, with high probability both algorithms converge to a small neighborhood around a stable point. Recall that $m$ is the dimension of data samples $z$.

\begin{theorem}
\label{theorem:finite_samples}
Suppose that the loss $\ell(z;\theta)$ is $\beta$-jointly smooth \eqref{ass:a2} and $\gamma$-strongly convex \eqref{ass:a3}, and that there exist $\alpha>1,\mu>0$ such that $\xi_{\alpha,\mu} \defeq \int_{\R^m} e^{\mu |x|^\alpha} d\D(\theta)$ is finite  $\forall \theta\in\Theta$. Let $\delta\in(0,1)$ be a radius of convergence.  Consider running RERM or RGD with $n_t = O\left(\frac{1}{(\epsilon \delta)^m}\log\big(\frac t p\big)\right)$ samples at time $t$.

\begin{itemize}[nolistsep]
	\item[(a)] If $\D(\cdot)$ is $\epsilon$-sensitive with $\epsilon<\frac{\gamma}{2\beta}$, then with probability $1-p$, RERM satisfies,
\begin{equation*}\|\theta_t - \thetaPS\|_2 \leq \delta, \text{ for all } t\geq \frac{\log\left(\frac 1 \delta\|\theta_0 - \thetaPS\|_2\right)}{\left(1 - \frac{2\epsilon\beta}{\gamma}\right)} .
\end{equation*}
\item[(b)]If  $\D(\cdot)$ is $\epsilon$-sensitive with $\epsilon< \frac{\gamma}{(\beta + \gamma)(1 + 1.5\eta\beta)}$, then with probability $1-p$, REGD with satisfies,
$$\|\theta_t - \thetaPS\|_2 \leq \delta, \text{ for all } t\geq \frac{\log\left(\frac 1 \delta\|\theta_0 - \thetaPS\|_2\right)}{\eta \left(\frac{\beta\gamma}{\beta + \gamma} - \epsilon(3\eta\beta^2 + 2 \beta)\right)},$$
for a constant choice of step size $\eta \leq \frac{2}{\beta + \gamma}.$
\end{itemize}
\end{theorem}

\begin{proof}[Proof sketch.]
The basic idea behind these results is the following. While $\|\theta_t- \thetaPS\|_2>\delta$, the sample size $n_t$ is sufficiently large to ensure a behavior similar to that on a population level: as in Theorems \ref{theorem:exact_min_strongly_convex} and \ref{theorem:contraction}, the iterates $\theta_t$ contract toward $\thetaPS$. This implies that $\theta_t$ eventually enters a $\delta$-ball around $\thetaPS$, for some large enough $t$. Once this happens, contrary to population-level results, a contraction is no longer guaranteed due to the noise inherent in observing only finite-sample approximations of $\D(\theta_t)$. Nevertheless, the sample size $n_t$ is sufficiently large to ensure that $\theta_t$ cannot escape a $\delta$-ball around $\thetaPS$ either.
\end{proof}


\section{Relating performative optimality and stability}

As we discussed previously, while performative optima are always guaranteed to exist,\footnote{In particular, they are guaranteed to exist over the extended real line, i.e. we allow $\theta\in(\R\cup \{\pm\infty\})^d$.} it is not clear whether performatively stable points exist in all settings. Our algorithmic analysis of repeated risk minimization and repeated gradient descent revealed the existence of unique stable points under the assumption that the objective is strongly convex and smooth. The first result of this section illustrates existence of stable points under weaker assumptions on the loss, in the case where the solution space $\Theta$ is constrained. All proofs can be found in Appendix \ref{app:mainproofs}.

\begin{proposition}
\label{propo:existence_stable_points}
	Let the distribution map $\D(\cdot)$ be $\epsilon$-sensitive and $\Theta \subset \R^d$ be compact. If the loss $\ell(z;\theta)$ is convex and jointly continuous in $(z,\theta)$, then there exists a performatively stable point.
\end{proposition}

A natural question to consider at this point is whether there are procedures analogous to RRM and RGD for efficiently computing performative optima. 

Our analysis suggests that directly minimizing the performative risk is in general a more challenging problem than finding performatively stable points. In particular, we can construct simple examples where the performative risk $\PR(\theta)$ is non-convex, despite strong regularity assumptions on the loss and the distribution map. 

\begin{proposition}
\label{prop:propoconcave}
	The performative risk $\PR(\theta)$ can be concave in $\theta$, even if the loss $\ell(z;\theta)$ is $\beta$-jointly smooth \eqref{ass:a2}, $\gamma$-strongly convex \eqref{ass:a3}, and the distribution map $\D(\cdot)$ is $\epsilon$-sensitive with $\epsilon < \frac{\gamma}{\beta}$.
\end{proposition}

However, what we can show is that there are cases where finding performatively stable points is sufficient to guarantee that the resulting model has low performative risk. In particular, our next result demonstrates that if the loss function $\ell(z;\theta)$ is Lipschitz in $z$ and $\gamma$-strongly convex, then all performatively stable points and performative optima lie in a small neighborhood around each other. Moreover, the theorem holds for cases where performative optima and performatively stable points are not necessarily unique.

\begin{theorem}
\label{theorem:closeness}
Suppose that the loss $\ell(z;\theta)$ is $L_z$-Lipschitz  in $z$, $\gamma$-strongly convex  \eqref{ass:a3}, and that the distribution map $\D(\cdot)$ is $\epsilon$-sensitive. Then, for every performatively stable point $\thetaPS$ and every performative optimum $\thetaPO$:
 $$\|\thetaPO- \thetaPS\|_2\leq \frac{2L_z\epsilon}{\gamma}.$$
\end{theorem}

This result shows that in cases where repeated risk minimization converges to a stable point, the resulting model approximately minimizes the performative risk.

Moreover, Theorem \ref{theorem:closeness} suggests a way of converging close to performative optima \emph{in objective value} even if the loss function is smooth and convex, but not strongly convex. In particular, by adding quadratic regularization to the objective, we can ensure that RRM or RGD converge to a stable point that approximately minimizes the performative risk, see Appendix \ref{sec:regularization}.


\section{A case study in strategic classification}
\label{sec:exp}
Having presented our model for performative prediction, we now proceed to illustrate how these ideas can be applied within the context of strategic classification and discuss some of the implications of our theorems for this setting.

We begin by formally establishing how strategic classification can be cast as a performative prediction problem and illustrate how our framework can be used to derive results regarding the convergence of popular retraining heuristics in strategic classification settings. Afterwards, we further develop the connections between both fields by empirically evaluating the behavior of repeated risk minimization on a dynamic credit scoring task. 

\subsection{Stackelberg equilibria are performative optima}

Strategic classification is a two-player game between an institution which deploys a classifier and agents who selectively adapt their features in order to improve their outcomes. 

A classic example of this setting is that of a bank which uses a machine learning classifier to predict whether or not a loan applicant is creditworthy. Individual applicants react to the bank's classifier by manipulating their features with the hopes of inducing a favorable classification. This game is said to have a \emph{Stackelberg} structure since agents adapt their features only after the bank has deployed their classifier.

The optimal strategy for the institution in a strategic classification setting is to deploy the solution corresponding to the \emph{Stackelberg equilibrium}, defined as the classifier $f_\theta$ which achieves minimal loss over the induced distribution $\D(\theta)$ in which agents have strategically adapted their features in response to $f_\theta$. In fact, we see that this equilibrium notion exactly matches our definition of performative optimality: 
\begin{equation*}
f_{\theta_{\mathrm{SE}}} \text{ is a Stackelberg equilibrium} \iff \theta_{\mathrm{SE}} \in \argmin_\theta \PR(\theta).
\end{equation*}

We think of $\D$ as a "baseline" distribution over feature-outcome pairs before any classifier deployment, and $\D(\theta)$ denotes the distribution over features and outcomes obtained by strategically manipulating $\D$. As described in existing work \cite{hardt2016strategic, milli2019social, bruckner2012static}, the distribution function $\D(\theta)$ in strategic classification corresponds to the data-generating process outlined in Figure~\ref{fig:distmap}.
\begin{figure}[t]
\setlength{\fboxsep}{2mm}
\begin{center}
\begin{boxedminipage}{\textwidth}
\noindent {\bf Input:} base distribution $\D$, classifier $f_\theta$, cost function $c$, and utility function $u$ \\
\noindent {\bf Sampling procedure for $\D(\theta)$:}
\begin{enum}
\item Sample $(x,y) \sim \D$
\item Compute best response $x_{\mathrm{BR}} \leftarrow \argmax_{x'} u(x', \theta) - c(x',x)$
\item Output sample $(x_{\mathrm{BR}}, \;y)$ 
\end{enum}
\end{boxedminipage}
\end{center}
\vspace{-3mm}
\caption{Distribution map for strategic classification.}
\label{fig:distmap} 
\end{figure}

Here, $u$ and $c$ are problem-specific functions which determine the best response for agents in the game. Together with the base distribution $\D$, these define the relevant distribution map $\D(\cdot)$ for the problem of strategic classification. 

A strategy that is commonly adapted in practice as a means of coping with the distribution shift that arises in strategic classification is to repeatedly retrain classifiers on the induced distributions. This procedure corresponds to the repeated risk minimization procedure introduced in Definition \ref{def:rrm}. Our results describe the first set of sufficient conditions under which repeated retraining overcomes strategic effects.

\begin{corollary}
\label{corollary:strategic_classification}
	Let the institution's loss $\ell(z;\theta)$ be $L_z$- and $L_\theta$-Lipschitz in $z$ and $\theta$ respectively, $\beta$-jointly smooth  \eqref{ass:a2}, and $\gamma$-strongly convex \eqref{ass:a3}. If the induced distribution map is $\epsilon$-sensitive, with $\epsilon <\frac \gamma  \beta$, then RRM converges at a linear rate to a performatively stable classifier $\thetaPS$ that is $2L_z \epsilon (L_\theta + L_z\epsilon) \gamma^{-1}$ close in objective value to the Stackelberg equilibrium. 

\end{corollary}

\subsection{Simulations}

We next examine the convergence of repeated risk minimization and repeated gradient descent in a simulated strategic classification setting. We run experiments on a dynamic credit scoring simulator in which an institution classifies the creditworthiness of loan applicants.\footnote{Code is available at \href{https://github.com/zykls/performative-prediction}{https://github.com/zykls/performative-prediction}, and the simulator has been integrated into the WhyNot software package~\cite{miller2020whynot}.} As motivated previously, agents react to the institution's classifier by manipulating their features to increase the likelihood that they receive a favorable classification.

To run our simulations, we construct a distribution map $\D(\theta)$, as described in Figure~\ref{fig:distmap}. For the base distribution $\D$, we use a class-balanced subset of a Kaggle credit scoring dataset \cite{creditdata}. Features $x\in\R^{m-1}$ correspond to historical information about an individual, such as their monthly income and number of credit lines. Outcomes $y\in\{0,1\}$ are binary variables which are equal to 1 if the individual defaulted on a loan and 0 otherwise.  

\begin{figure}[t]
  \centering
        \includegraphics[width=0.48\linewidth]{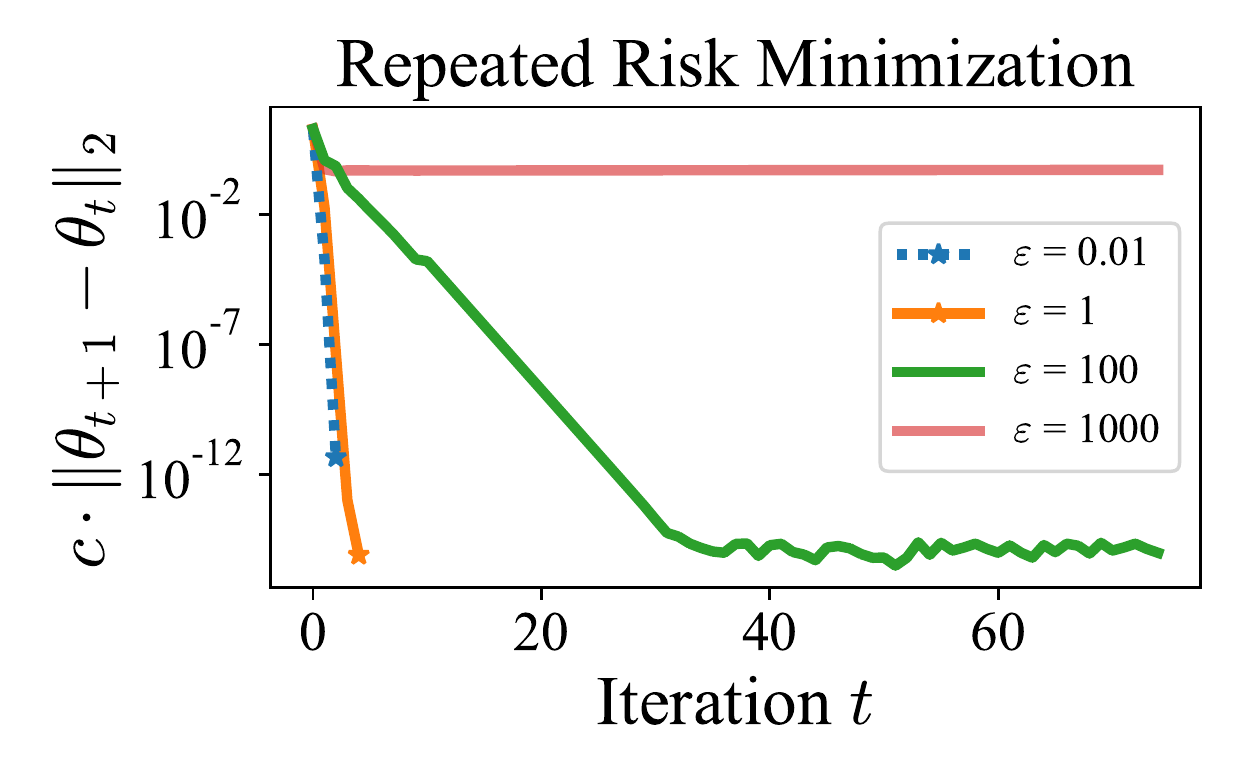}~~~
        \includegraphics[width=0.48\linewidth]{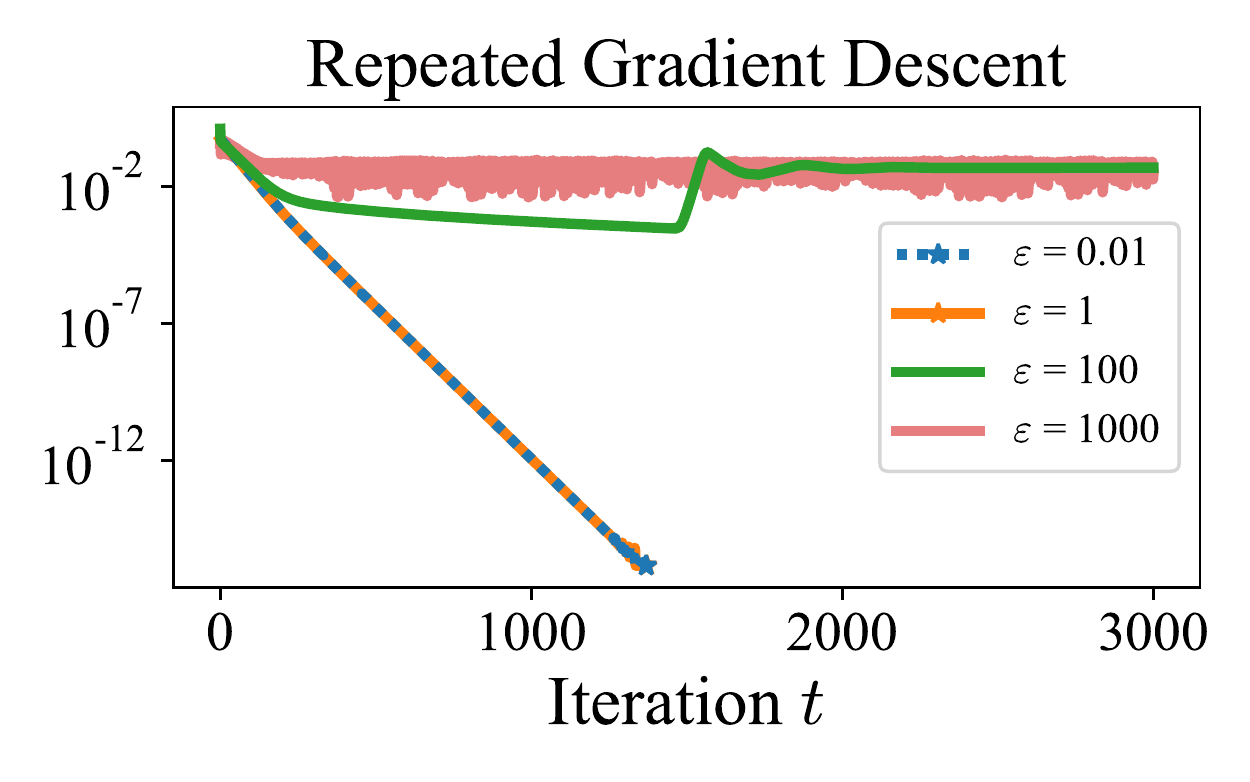}
\caption{Convergence in domain of RRM (left) and RGD (right) for varying $\epsilon$-sensitivity parameters. We add a marker if at the next iteration the distance between iterates is numerically zero. We normalize the distance by $c=\|\theta_{0,S}\|_2^{-1}$.}
\label{fig:minimization}
\end{figure}	
 
The institution makes predictions using a logistic regression classifier. We assume that individuals have linear utilities $u(\theta, x) = -\langle \theta, x \rangle$ and quadratic costs $c(x', x) = \frac{1}{2\epsilon}\norm{x'-x}^2$, where $\epsilon$ is a positive constant that regulates the cost incurred by changing features. Linear utilities indicate that agents wish to minimize their assigned probability of default. 

We divide the set of features into strategic features $S \subseteq [m-1]$, such as the number of open credit lines, and non-strategic features (e.g., age). Solving the optimization problem described in  Figure~\ref{fig:distmap}, the best response for an individual corresponds to the following update, 
\begin{equation*}
x'_S = x_S -\epsilon\theta_S,
\end{equation*}
where $x_S,x'_S, \theta_S\in\R^{|S|}$. As per convention in the literature \cite{bruckner2012static, hardt2016strategic, milli2019social}, individual outcomes $y$ are unaffected by strategic manipulation. 

Intuitively, this data-generating process is $\epsilon$-sensitive since for a given choice of classifiers, $f_\theta$ and $f_{\theta'}$, an individual feature vector is shifted to  $x_S - \epsilon \theta_S$ and to $x_S-\epsilon\theta'_S$, respectively. The distance between these two shifted points is equal to $\epsilon\|\theta_S - \theta'_S\|_2$. Since the optimal transport distance is bounded by $\epsilon\|\theta - \theta'\|_2$ for every individual point, it is also bounded by this quantity over the entire distribution. A full proof of this claim is presented in Appendix~\ref{sec:exp_details}.

For our experiments, instead of sampling from $\D(\theta)$, we treat the points in the original dataset as the true distribution. Hence, we can think of all the following procedures as operating at the population level. Furthermore, we add a regularization term to the logistic loss to ensure that the objective is strongly convex.

\paragraph{Repeated risk minimization.} The first experiment we consider is the convergence of RRM. From our theoretical analysis, we know that RRM is guaranteed to converge at a linear rate to a performatively stable point if the sensitivity parameter $\epsilon$ is smaller than $\frac{\gamma}{\beta}$. In Figure~\ref{fig:minimization} (left), we see that RRM does indeed converge in only a few iterations for small values of $\epsilon$ while it divergences if $\epsilon$ is too large.

\begin{figure}[t]
  \centering
        \includegraphics[width=0.45\linewidth]{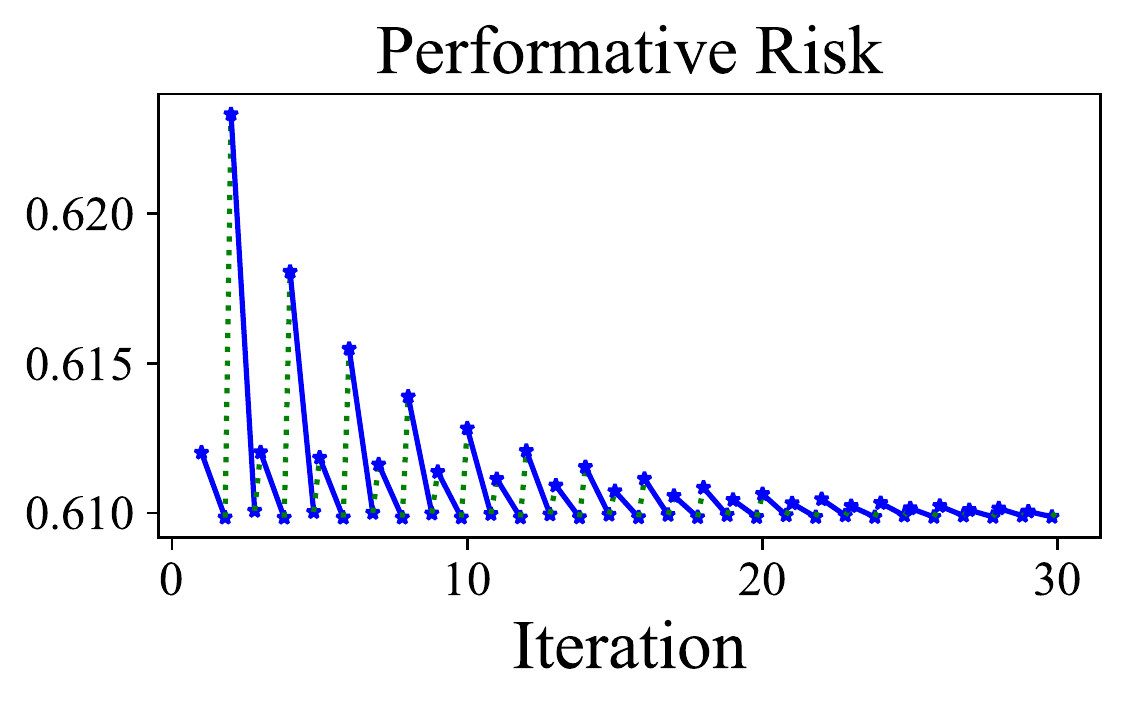}~~~
        \includegraphics[width=0.45\linewidth]{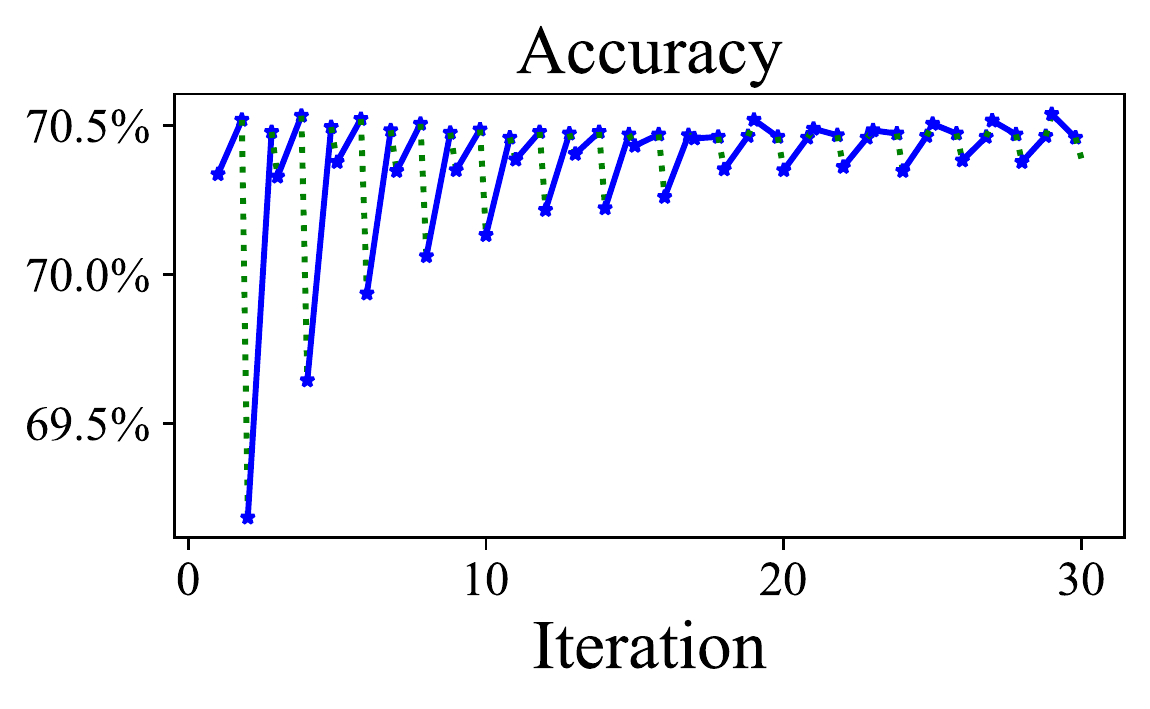}
\caption{Performative risk (left) and accuracy (right) of the classifier $\theta_t$ at different stages of RRM for $\epsilon=80$. Blue lines indicates the optimization phase and green lines indicate the effect of the distribution shift after the classifier deployment.  }
\label{fig:zigzag}
\end{figure}

The evolution of the performative risk during the RRM optimization is illustrated in Figure~\ref{fig:zigzag}. We evaluate $\PR(\theta)$ at the beginning and at the end of each optimization round and indicate the effect due to distribution shift with a dashed green line. We also verify that  the surrogate loss is a good proxy for classification accuracy in the performative setting.

\paragraph{Repeated gradient descent.} In the case of RGD, we find similar behavior to that of RRM. While the iterates again converge linearly, they naturally do so at a slower rate than in the exact minimization setting, given that each iteration consists only of a single gradient step. 
Again, we can see in Figure~\ref{fig:minimization} that the iterates converge for small values of $\epsilon$ and diverge for large values.


\section{Discussion and Future Work}

Our work draws attention to the fundamental problem of performativity in statistical learning and decision-making. Performative prediction enjoys a clean formal setup that we introduced, drawing on elements from causality and game theory.

Retraining is often considered a nuisance intended to cope with distribution shift. In contrast, our work interprets retraining as the natural equilibrating dynamic for performative prediction. The fixed points of retraining are performative stable points. Moreover, retraining converges to such stable points under natural assumptions, including strong convexity of the loss function. It is interesting to note that (weak) convexity alone is not enough. Performativity thus gives another intriguing perspective on why strong convexity is desirable in supervised learning.

Several interesting questions remain. For example, by letting the step size of repeated gradient descent tend to $0$, we see that this procedure converges for $\epsilon < \frac{\gamma}{\beta + \gamma}$. Exact repeated risk minimization, on the other hand, provably converges for every $\epsilon <\frac{\gamma}{\beta}$, and we showed this inequality is tight. It would be interesting to understand whether this gap is a fundamental difference between both procedures or an artifact of our analysis.

Lastly, we believe that the tools and ideas from performative prediction can be used to make progress in other subareas of machine learning. For example, in this paper, we have illustrated how reframing strategic classification as a performative prediction problem leads to a new understanding of when retraining overcomes strategic effects. However, we view this example as only scratching the surface of work connecting performative prediction with other fields.

In particular, reinforcement learning can be thought of as a case of performative prediction. In this setting, the choice of policy $f_\theta$, affects the distribution $\D(\theta)$ over $z = \{(s_h, a_h)\}_{h=1}^\infty$, the set of visited states, $s$, and actions, $a$, in a Markov Decision Process. Building off this connection, we can reinterpret repeated risk minimization as a form of off-policy learning in which an agent first collects a batch of data under a particular policy $f_\theta$, and then finds the optimal policy for that trajectory offline. We believe that some of the ideas developed in the context of performative prediction can shed new light on when these off-policy methods can converge.

\section*{Acknowledgements}
We wish to acknowledge support from the U.S. National Science Foundation Graduate Research Fellowship Program and the Swiss National Science Foundation Early Postdoc Mobility Fellowship Program.

\bibliographystyle{plain}
\bibliography{refs}

\newpage
\appendix

\section{Applications of performativity}

To illustrate the fact that performativity is a common cause for concept drift, we review a table of concepts drift applications from \cite{zliobaite2016}. In Table \ref{tab:my-table}, we highlight those settings that naturally occur due to performativity. Below we briefly discuss the role of performativity in such applications.

\begin{table}[h]
\small
\centering
\makebox[\linewidth]{
\begin{tabular}{|l|l|l|l|}
\hline
\backslashbox{\textbf{Indust.}}{\textbf{Appl.}}                                                                                                                 &\textbf{Monitoring \& control}                                                                                                          & \textbf{Information management}                                                                                                                                                    & \textbf{Analytics \& diagnostics}                                                                                          \\ \hline
\textbf{\begin{tabular}[c]{@{}l@{}}Security, \\ Police\end{tabular}}                                                              & \cellcolor[HTML]{E4E9F9}\begin{tabular}[c]{@{}l@{}}fraud detection,\\  insider trading detection, \\ adversary actions detection\end{tabular} & \cellcolor[HTML]{E4E9F9}\begin{tabular}[c]{@{}l@{}}next crime place \\ prediction\end{tabular}                                                                                     & \begin{tabular}[c]{@{}l@{}}crime volume \\ prediction\end{tabular}                                                          \\ \hline
\textbf{\begin{tabular}[c]{@{}l@{}}Finance,\\ Banking,\\ Telecom,\\ Insurance,\\ Marketing,\\ Retail,\\ Advertising\end{tabular}} & \begin{tabular}[c]{@{}l@{}}monitoring \& management \\ of customer segments, \\ bankruptcy prediction\end{tabular}                            & \cellcolor[HTML]{E4E9F9}\begin{tabular}[c]{@{}l@{}}product or service\\  recommendation, \\ including complimentary, \\ user intent or information \\ need prediction\end{tabular} & \begin{tabular}[c]{@{}l@{}}demand prediction, \\ response rate \\ prediction, budget \\ planning\end{tabular}               \\ \hline
\textbf{\begin{tabular}[c]{@{}l@{}}Production\\ industry\end{tabular}}                                                            & controlling output quality                                                                                                                    & -                                                                                                                                                                                  & predict bottlenecks                                                                                                         \\ \hline
\textbf{\begin{tabular}[c]{@{}l@{}}Education \\(e-Learning,\\ e-Health), \\ Media, \\ Entertainment\end{tabular}}                & \cellcolor[HTML]{E4E9F9}\begin{tabular}[c]{@{}l@{}}gaming the system,\\  drop out prediction\end{tabular}                                     & \cellcolor[HTML]{E4E9F9}\begin{tabular}[c]{@{}l@{}}music, VOD, movie, \\ news, learning object \\ personalized search\\  \& recommendations\end{tabular}                           & \cellcolor[HTML]{E4E9F9}\begin{tabular}[c]{@{}l@{}}player-centered game\\ design, learner-centered\\ education\end{tabular} \\ \hline
\end{tabular}}
\caption{Table of concept drift applications from \v{Z}liobait{\.{e}} et al. \cite{zliobaite2016}}
\label{tab:my-table}
\end{table}

The role of \emph{fraud detection} systems is to predict whether an instance such as a transaction or email is legitimate or not. It is well-known that designers of such fraudulent instances adapt to the fraud detection system in place in order to breach security \cite{bolton2002statistical}. Therefore, the deployment of fraud detection systems shapes the features of fraudulent instances.

\emph{Crime place prediction}, sometimes referred to as predictive policing \cite{feedbackPredictivePolicing, kubeAllocation, predictAndServe}, uses historical data to estimate the likelihood of crime at a given location. Those locations where criminal behavior is deemed likely by the system typically get more police patrols and better surveillance which act on the one hand to deter crime. These actions resulting from prediction significantly decrease the probability of crime taking place, thus changing the data used for future predictions.

In \emph{personalized recommendations}, instances are recommended to a user based on their historical context, such as their ratings or purchases. The set of recommendations thus depends on the trained machine learning model, which in turn changes the user's future ratings or purchases \cite{bottou2013counterfactual}. In other words, user features serving as input to a recommender inevitably depend on the previously used recommendation mechanisms.

In \emph{online two-player games}, it is common to request an AI opponent. The level of sophistication of the AI opponent might be chosen depending on the user's success history in the given game, with the goal of making the game appropriately challenging. This choice of AI opponent changes players' future success profiles, again causing a distribution shift in the features serving as an input to the prediction system.

\emph{Gaming the system} falls under the umbrella of strategic classification, which we discuss in detail in Section \ref{sec:exp}, so we avoid further discussion in this section.

\section{Experiments}

\subsection{Visualizing the performative risk and trajectory of RRM}

\definecolor{brass}{rgb}{0.71, 0.65, 0.26}

\begin{figure}[h]
        \centering
        \subfigure[$\epsilon=25$]{\includegraphics[width=0.8\linewidth]{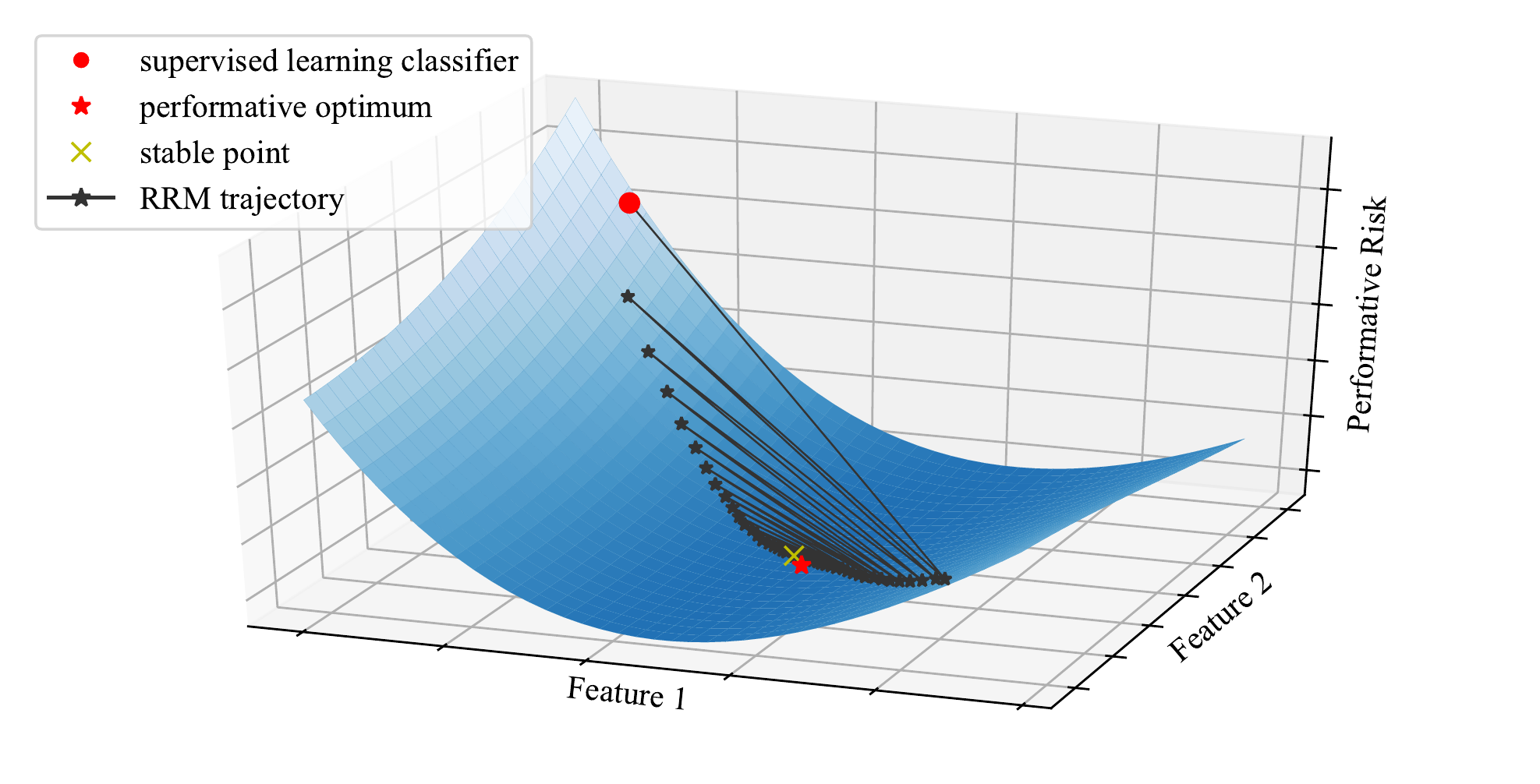}}
\subfigure[$\epsilon=100$]{ \includegraphics[width=0.8\linewidth]{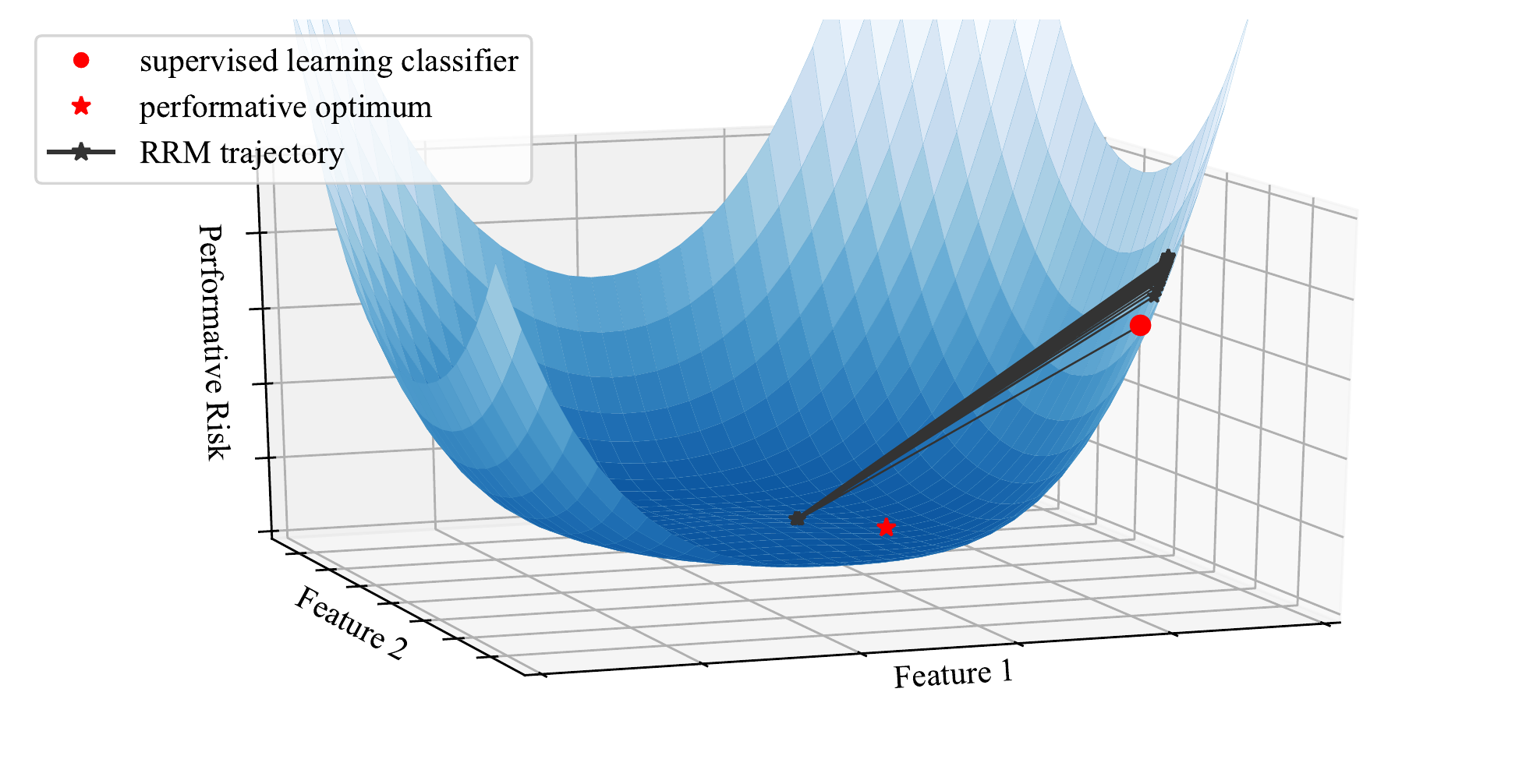}}
\caption{Performative risk surface and trajectory of repeated risk minimization for two different values of sensitivity parameter $\epsilon$. The initial iterate is the risk minimizer on the base dataset ($\color{red}\bullet$). We mark the performative optimum ($\color{red}\star$) and performatively stable point ($\color{brass}\boldsymbol\times$).}
\label{fig:trajectory}
\end{figure}

We provide additional experimental results in which we visualize the trajectory of repeated risk minimization on the surface of the performative risk. We adopt the general setting of Section \ref{sec:exp}. However, to properly visualize the loss, we rerun the experiments on a reduced version of the dataset with only two features (i.e $x \in \R^2$), both of which are adapted strategically according to the update described in Section \ref{sec:exp}. 

Figure~\ref{fig:trajectory} plots the performative risk surface, together with the trajectory of RRM given by straight black lines. The top plot shows the trajectory for a suitably small sensitivity parameter $\epsilon$. We see that RRM converges to a stable point which is close to the performative optimum. 
We contrast this behavior with that of RRM when $\epsilon$ is large in the bottom plot. Here, we observe that the iterates oscillate and that the algorithm fails to converge. 
Both plots mark the risk minimizer on the initial data set ($\color{red}\bullet$), before any strategic adaptation takes place. This point also corresponds to the initial iterate of RRM $\theta_0$. We additionally mark the performative optimum ($\color{red}\star$) on the risk curve. The top plot additionally marks the last iterate of RRM, which serves as a proxy for the performatively stable point ($\color{brass}\boldsymbol\times$). As predicted by our theory, this stable point is in a small neighborhood around the performative optimum.

\subsection{Experimental details}
\label{sec:exp_details}

\paragraph{Base distribution.} The base distribution consists of the Kaggle data set \cite{creditdata}. We subsample $n=18,357$ points from the original training set such that both classes are approximately balanced (45\% of points have $y$ equal to 1). There are a total of 10 features, 3 of which we treat as strategic features: utilization of credit lines, number of open credit lines, and number of real estate loans. We scale features in the base distribution so that they have zero mean and unit variance.

\paragraph{Verifying $\epsilon$-sensitivity.}
We verify that the map $\D(\cdot)$, as described in Section \ref{sec:exp}, is $\epsilon$-sensitive. To do so, we analyze $W_1(\D(\theta),\D(\theta'))$, for arbitrary $\theta,\theta'\in\Theta$. Fix a sample point $x\in\R^{m-1}$ from the base dataset. Because the base distribution $\D$ is supported on $n$ points, we can upper bound the optimal transport distance between any pair of distributions $\D(\theta)$ and $\D(\theta')$ by the Euclidean distance between the shifted versions of $x$ in $\D(\theta)$ and $\D(\theta')$. 

In our construction, the point $x$ is shifted to $x - \epsilon \theta$ and to $x-\epsilon\theta'$ in $\D(\theta)$ and $\D(\theta')$ respectively. The distance between these two shifted points is $\|x - \epsilon \theta - x + \epsilon \theta'\|_2 = \epsilon\|\theta - \theta'\|_2$. Since the same relationship holds for all other samples $x$ in the base dataset, the optimal transport from $\D(\theta)$ to $\D(\theta')$ is at most $\epsilon\|\theta - \theta'\|_2$.

\paragraph{Verifying joint smoothness of the objective.}

For the experiments described in Figure \ref{fig:minimization}, we run repeated risk minimization and repeated gradient descent on the logistic loss with $\ell_2$ regularization:
\begin{equation}
\label{eq:loss_objective}
	\frac{1}{n}\sum_{i=1}^n -y_i \theta^\top x_i  + \log\big(1 + \exp( \theta^\top x_i)\big) + \frac{\gamma}{2}\norm{\theta}^2
\end{equation}

For both the repeated risk minimization and repeated gradient descent we set $\gamma = 1000/n$, where $n$ is the size of the base dataset.

For a particular feature-outcome pair $(x_i,y_i)$, the logistic loss is $\frac{1}{4}\norm{x_i}^2$ smooth \cite{understandingML}. Therefore, the entire objective is $\frac{1}{4n}\sum_{i=1}^n \norm{x_i}^2 + \gamma$ smooth. Due to the strategic updates, $x_{BR} = x - \epsilon \theta$, the norm of individual features change depending on the choice of model parameters. 

Theoretically, we can upper bound the smoothness of the objective by finding the implicit constraints on $\Theta$, which can be revealed by looking at the dual of the objective function for every fixed value of $\epsilon$. However, for simplicity, we simply calculate the worst-case smoothness of the objective, given the trajectory of iterates $\{\theta_t\}$, for every fixed $\epsilon$. 

Furthermore, we can verify the logistic loss is jointly smooth. For a fixed example $z=(x,y)$, the gradient of the regularized logistic loss with respect to $\theta$ is,

\[
\nabla_\theta \ell(z;\theta)= y x + \frac{\exp(\theta^\top x)}{1 + \exp(\theta^\top x)}x  + \gamma \theta,
\]
which is 2-Lipschitz in $z$ due to $y\in\{0,1\}$. Hence, the overall objective is $\beta$-jointly smooth with parameter $$\beta = \max\big\{2, \frac{1}{4n}\sum_{i=1}^n \norm{x_i}^2 + \gamma\big\}.$$

For RRM, $\epsilon$ is less than $\frac{\gamma}{\beta}$ only in the case that $\epsilon=0.01$. For RGD, $\epsilon$ is never smaller than the theoretical cutoff of $\frac{\gamma}{(\beta + \gamma)(1 + 1.5 \eta \beta)}$.

\paragraph{Optimization details.}

The definition of RRM requires exact minimization of the objective at every iteration. We approximate this requirement by minimizing the objective described in expression \eqref{eq:loss_objective} to small tolerance, $10^{-8}$, using gradient descent. We choose the step size at every iteration using backtracking line search.

In the case of repeated gradient descent, we run the procedure as described in Definition \ref{def:rgd} with a fixed step size of $\eta = \frac{2}{\beta + \gamma}$.

\section{Auxiliary lemmas}

\begin{lemma}[First-order optimality condition]
\label{lemma:first_order_opt_condition}
	Let $f$ be convex and let $\Omega$ be a closed convex set on which $f$ is differentiable, then 
	\[
	x_* \in \argmin_{x \in \Omega} f(x)
	\]
	if and only if 
	\[
	\nabla f(x_*)^T(y- x_*) \geq 0, \quad \forall y \in \Omega.
	\]
\end{lemma}

\begin{lemma}[Bubeck, 2015 \cite{bubeck2015convex}, Lemma 3.11]
\label{lemma:smooth_strongly_convex}
Let $f:\R^d\rightarrow\R$ be $\beta$-smooth and $\gamma$-strongly convex, then for all $x,y$ in $\R^d$, 
\[
(\nabla f(x) - \nabla f(y))^\top (x - y) \geq \frac{\gamma \beta}{\gamma + \beta}\norm{x-y}^2 + \frac{1}{\gamma + \beta}\norm{\nabla f(x) - \nabla f(y)}^2.
\]
	
\end{lemma}

\begin{lemma}[Kantorovich-Rubinstein]
\label{lemma:duality} A distribution map $\D(\cdot)$ is $\epsilon$-sensitive if and only if for all $\theta,\theta'\in\Theta$:
\begin{equation*}
\sup \left\{\Big|~ \E_{Z\sim \D(\theta)}g(Z) - \E_{Z\sim \D(\theta')}g(Z) ~\Big| \leq \epsilon\|\theta -\theta'\|_2~:~g:\R^p\rightarrow\R, ~g \text{ 1-Lipschitz}\right\}.
\end{equation*}
\end{lemma}

\begin{lemma}
\label{lemma:wasserstein_application}
Let $f:\R^n\rightarrow \R^d$ be an $L$-Lipschitz function, and let $X, X'\in\R^n$ be random variables such that $W_1(X,X')\leq C$. Then
$$\|\E[f(X)] -\E[f(X')]\|_2 \leq LC.$$
\end{lemma}

\begin{proof}
\begin{align*}
	\|\E[f(X)] -\E[f(X')]\|_2^2 &= (\E[f(X)] -\E[f(X')])^\top (\E[f(X)] -\E[f(X')])\\
	&= \|\E[f(X)] -\E[f(X')]\|_2 \frac{(\E[f(X)] -\E[f(X')])^\top}{\|\E[f(X)] -\E[f(X')]\|_2} (\E[f(X)] -\E[f(X')]).
\end{align*}
Now define the unit vector $v:=\frac{\E[f(X)] -\E[f(X')]}{\|\E[f(X)] -\E[f(X')]\|_2}$. By linearity of expectation, we can further write
$$\|\E[f(X)] -\E[f(X')]\|_2^2 = \|\E[f(X)] -\E[f(X')]\|_2~ (\E[v^\top f(X)] - \E[v^\top f(X')]).$$
For any unit vector $v$ and $L$-Lipschitz function $f$, $v^\top f$ is a one-dimensional $L$-Lipschitz function, so we can apply Lemma \ref{lemma:duality} to obtain
$$\|\E[f(X)] -\E[f(X')]\|_2^2\leq \|\E[f(X)] -\E[f(X')]\|_2 LC.$$
Canceling out $\|\E[f(X)] -\E[f(X')]\|_2$ from both sides concludes the proof.
\end{proof}

\section{Proofs of main results}
\label{app:mainproofs}

\subsection{Proof of Theorem \ref{theorem:exact_min_strongly_convex}}
\label{app:proofexact}

Fix $\theta,\theta'\in\Theta$. Let $f(\varphi)=  \ploss{\theta}{\varphi}$ and $f'(\varphi)=  \ploss{\theta'}{\varphi}$. Since $f$ is $\gamma$-strongly convex and $G(\theta)$  is the unique minimizer of $f(x)$  we know that,
\begin{align}
f(G(\theta))-f(G(\theta'))&\geq \left(G(\theta)-G(\theta')\right)^\top\nabla f(G(\theta')) + \frac{\gamma}{2} \|G(\theta)-G(\theta')\|_2^2\\
f(G(\theta'))-f(G(\theta))&\geq \frac{\gamma}{2} \|G(\theta)-G(\theta')\|_2^2
\end{align}
Together, these two inequalities imply that 
\[
-\gamma \|G(\theta)-G(\theta')\|_2^2 \geq \left(G(\theta)-G(\theta')\right)^\top\nabla f(G(\theta')).
\]

Next, we observe that $(G(\theta) - G({\theta'}))^\top\nabla_\theta \ell(z; G(\theta'))$ is $\|G(\theta) - G({\theta'})\|_2 \beta$-Lipschitz in $z$. This follows from applying Cauchy-Schwarz and the fact that the loss is $\beta$-jointly smooth. Using the dual formulation of the optimal transport distance (Lemma \ref{lemma:duality}) and $\epsilon$-sensitivity of $\D(\cdot)$,
\[
(G(\theta) - G({\theta'}))^\top  \nabla f(G(\theta')) -  (G(\theta) - G({\theta'}))^\top \nabla f'(G(\theta')) \geq - \epsilon\beta\|G(\theta)-G({\theta'})\|_2\|\theta-\theta'\|_2.
\]
Furthermore, using the first-order optimality conditions for convex functions, we have $(G(\theta) - G({\theta'}))^\top  \nabla f'(G(\theta'))\geq 0$, and hence $(G(\theta) - G({\theta'}))^\top  \nabla f(G(\theta')) \geq - \epsilon\beta\|G(\theta)-G({\theta'})\|_2\|\theta-\theta'\|_2.$ Therefore, we conclude that,
\[
-\gamma\|G(\theta)-G(\theta')\|_2^2 \geq  - \epsilon\beta\|G(\theta)-G({\theta'})\|_2\|\theta-\theta'\|_2.
\]
Claim (a) then follows by rearranging. 

To prove claim (b) we note that  $\theta_t = G({\theta_{t-1}})$ by the definition of RRM, and $G({\thetaPS}) = \thetaPS$ by the definition of stability. Applying the result of part (a) yields
\begin{equation}\|\theta_{t} - \thetaPS\|_2 \leq \epsilon \frac{\beta}{\gamma} \|\theta_{t-1} - \thetaPS\|_2 \leq \left(\epsilon \frac{\beta}{\gamma}\right)^t \|\theta_{0} - \thetaPS\|_2.
\end{equation}
Setting this expression to be at most $\delta$ and solving for $t$ completes the proof of claim (b).

\subsection{Proof of Proposition \ref{prop:tightness}}
\label{app:prooftightness}

As for statement (a), we provide one counterexample for each of the statements (b) and (c).

\paragraph{Proof of (b):} Consider a type of regularized hinge loss $\ell(z;\theta) = C \max(-1,y\theta) + \frac{\gamma}{2} (\theta - 1)^2$, and suppose $\Theta \supseteq [-\frac{1}{2\epsilon}, \frac{1}{2\epsilon}]$.

Let the distribution of $Y$ according to $\D(\theta)$ be a point mass at $\epsilon \theta$, and let the distribution of $X$ be invariant with respect to $\theta$. Clearly, this distribution is $\epsilon$-sensitive.

 Let $\theta_0 = 2$. Then, by picking $C$ big enough, RRM prioritizes to minimize the first term exactly, and hence we get $\theta_1 = -\frac{1}{2\epsilon}$. In the next step, again due to large $C$, we get $\theta_2 = 2$. Thus, RRM keeps oscillating between $2$ and $-\frac{1}{2\epsilon}$, failing to converge. This argument holds for all $\gamma, \epsilon >0$.

\paragraph{Proof of (c):} Suppose that the loss function is the squared loss, $\ell(z;\theta) = (y-\theta)^2$, where $y,\theta\in\R$. Note that this implies $\beta = \gamma$. Let the distribution of $Y$ according to $\D(\theta)$ be a point mass at $1 + \epsilon \theta$, and let the distribution of $X$ be invariant with respect to $\theta$. This distribution family satisfies $\epsilon$-sensitivity, because
	$$W_1(\D(\theta), \D(\theta')) = \epsilon |\theta-\theta'|.$$
	  By properties of the squared loss, we know
	$$\argmin_{\theta'}  \map(\theta,\theta') = \E_{Z\sim \D(\theta)}\left[Y\right] = 1 +  \epsilon \theta.$$
It is thus not hard to see that RRM does not contract if $\epsilon\geq  \frac{\gamma}{\beta} = 1$:
$$|G(\theta) - G(\theta')| = \left|1 + \epsilon \theta - 1 - \epsilon \theta'\right| = \epsilon|\theta-\theta'|,$$
which exactly matches the bound of Theorem \ref{theorem:exact_min_strongly_convex} and proves the first statement of the proposition. The unique performatively stable point of this problem is $\theta$ such that $\theta = 1 + \epsilon\theta$, which is $\thetaPS =  \frac{1}{1 - \epsilon}$ for $\epsilon> 1$.  

For $\epsilon=1$, no performatively stable point exists, thereby proving the second claim of the proposition. If $\epsilon>1$ on the other hand, and $\theta_0\neq \thetaPS$, we either have $\theta_t \rightarrow \infty$ or $\theta_t \rightarrow - \infty$, because
$$\theta_{t} = 1 + \epsilon \theta_{t-1} = \sum_{k=0}^{t-1} \epsilon^k + \theta_0 \epsilon^t = \frac{\epsilon^t - 1}{\epsilon - 1} + \theta_0 \epsilon^t,$$
thus concluding the proof.

\subsection{Proof of Theorem \ref{theorem:contraction}}
\label{app:proofGD}

Since projecting onto a convex set can only bring two iterates closer together, in this proof we ignore the projection operator $\Pi_\Theta$ and treat $\gd$ as performing merely the gradient step.

We begin by expanding out $\norm{\gd(\theta) - \gd({\theta'})}^2$,
\begin{align*}
 \norm{\gd(\theta) - \gd({\theta'})}^2 &= \left\|\theta - \eta \lgrad{\theta}{\theta} - \theta' + \eta \lgrad{\theta'}{\theta'}\right\|_2^2\\
	&= \norm{\theta - \theta'}^2 - 2\eta (\theta-\theta')^\top\left(\lgrad{\theta}{\theta} - \lgrad{\theta'}{\theta'} \right)\\
	& \quad+ \eta^2 \left\|\lgrad{\theta}{\theta} - \lgrad{\theta'}{\theta'}\right\|_2^2\\
	&\defeq T_1 -2\eta T_2 + \eta^2 T_3.
\end{align*}
Next, we analyze each term individually,
\begin{align*}
	T_1 &\defeq \norm{\theta - \theta'}^2,\\
	T_2 &\defeq (\theta-\theta')^\top\left(\lgrad{\theta}{\theta} - \lgrad{\theta'}{\theta'} \right),\\
	T_3 &\defeq \norm{\lgrad{\theta}{\theta} - \lgrad{\theta'}{\theta'}}^2.
\end{align*}
We start by lower bounding $T_2$:
\begin{align*}
	T_2 &= (\theta-\theta')^\top\left(\lgrad{\theta}{\theta} - \lgrad{\theta'}{\theta}+ \lgrad{\theta'}{\theta} - \lgrad{\theta'}{\theta'} \right)\\
	&= (\theta-\theta')^\top\left(\lgrad{\theta}{\theta} - \lgrad{\theta'}{\theta}\right) + (\theta-\theta')^\top\left(\lgrad{\theta'}{\theta} - \lgrad{\theta'}{\theta'} \right)\\
	&\geq -\|\theta-\theta'\|_2 ~\|\lgrad{\theta}{\theta} - \lgrad{\theta'}{\theta}\|_2 + (\theta-\theta')^\top\left(\lgrad{\theta'}{\theta} - \lgrad{\theta'}{\theta'} \right),
\end{align*}
where in the last step we apply the Cauchy-Schwarz inequality.
By smoothness, $\nabla_\theta \ell(Z; \theta)$ is $\beta$-Lipschitz in $Z$. Together with the fact that $Z$ is $\epsilon$-sensitive, we can lower bound the first term in the above expression by applying Lemma \ref{lemma:wasserstein_application}, which results in $-\beta \epsilon \norm{\theta - \theta'}^2$. 

We apply Lemma \ref{lemma:smooth_strongly_convex} to lower bound the second term by 
\begin{align*}
& (\theta-\theta')^\top\left(\lgrad{\theta'}{\theta} - \lgrad{\theta'}{\theta'} \right) \\
\geq~ &\frac{\beta \gamma}{\beta + \gamma}\norm{\theta - \theta'}^2 + \frac{1}{\beta +\gamma} \E_{Z \sim \D(\theta')}\left[\norm{\nabla_\theta \ell(Z; \theta) - \nabla_\theta \ell(Z; \theta')}^2\right] \\ 
 \geq~ &\frac{\beta \gamma}{\beta + \gamma}\norm{\theta - \theta'}^2 + \frac{1}{\beta +\gamma} \left\|\E_{Z \sim \D(\theta')}\left[\nabla_\theta \ell(Z; \theta) - \nabla_\theta \ell(Z; \theta')\right]\right\|_2^2, 
\end{align*}
where we have applied Jensen's inequality in the last line. Putting everything together, we get
\[
T_2 \geq \left(\frac{\beta \gamma}{\beta + \gamma}  - \beta \epsilon\right)\norm{\theta - \theta'}^2 + \frac{1}{\beta +\gamma} \left\|\E_{Z \sim \D(\theta')}\left[\nabla_\theta \ell(Z; \theta) - \nabla_\theta \ell(Z; \theta')\right]\right\|_2^2.
\]
Now we upper bound $T_3$. We begin by expanding out the square just as before, 
\begin{align}
\begin{split}
\label{eq:t3}
T_3 &= \left\|\lgrad{\theta}{\theta} - \lgrad{\theta'}{\theta} + \lgrad{\theta'}{\theta} - \lgrad{\theta'}{\theta'}\right\|^2\\
& = \left\|\lgrad{\theta}{\theta} - \lgrad{\theta'}{\theta}\right\|_2^2 + \left\|\lgrad{\theta'}{\theta} - \lgrad{\theta'}{\theta'}\right\|_2^2	\\
&\quad+ 2\left(\lgrad{\theta}{\theta} - \lgrad{\theta'}{\theta}\right)^\top\left(\lgrad{\theta'}{\theta} - \lgrad{\theta'}{\theta'}\right).
\end{split}
\end{align}
We again bound each term individually. By the smoothness of the loss and Lemma \ref{lemma:wasserstein_application}, 
\[
\left\|\lgrad{\theta}{\theta} - \lgrad{\theta'}{\theta}\right\|_2^2 \leq \beta^2\epsilon^2 \|\theta-\theta'\|_2^2.
\]
Moving on to the last term in \eqref{eq:t3}:
\begin{align*}
 &2\left(\lgrad{\theta}{\theta} - \lgrad{\theta'}{\theta}\right)^\top\left(\lgrad{\theta'}{\theta} - \lgrad{\theta'}{\theta'}\right)\\
 \defeq~&2\left\|\lgrad{\theta'}{\theta} - \lgrad{\theta'}{\theta'}\right\|_2\left(\lgrad{\theta}{\theta} - \lgrad{\theta'}{\theta}\right)^\top v \\
 = ~&2\left\|\lgrad{\theta'}{\theta} - \lgrad{\theta'}{\theta'}\right\|_2 \left(\E_{Z \sim \D(\theta)} v^\top \nabla_\theta  \ell(Z;\theta) - \E_{Z \sim \D(\theta')}v^\top\nabla_\theta  \ell(Z;\theta)\right),
\end{align*}
where we define the unit vector $v\defeq \frac{\lgrad{\theta'}{\theta} - \lgrad{\theta'}{\theta'}}{ \norm{\lgrad{\theta'}{\theta} - \lgrad{\theta'}{\theta'}}}$. By smoothness of the loss, we can conclude that $v^\top\nabla_\theta  \ell(Z,\theta)$ is $\beta$-Lipschitz, so by $\epsilon$-sensitivity we get
\begin{align*}
 &2\left(\lgrad{\theta}{\theta} - \lgrad{\theta'}{\theta}\right)^\top\left(\lgrad{\theta'}{\theta} - \lgrad{\theta'}{\theta'}\right)\\
  \leq ~&2\left\|\lgrad{\theta'}{\theta} - \lgrad{\theta'}{\theta'}\right\|_2 \beta \epsilon \|\theta-\theta'\|_2 \\
  \leq ~&2 \beta^2 \epsilon \norm{\theta - \theta'}^2,
  \end{align*}
where in the last step we again apply smoothness. Hence, 
\[
T_3 \leq (\epsilon^2 \beta^2 + 2\beta^2\epsilon)\left\|\theta - \theta'\right\|_2^2 + \left\|\lgrad{\theta'}{\theta} - \lgrad{\theta'}{\theta'}\right\|_2^2.
\]
Having bounded all the terms, we now conclude that
\begin{align*}
	 \left\|\gd(\theta) - \gd({\theta'})\right\|_2^2 &\leq  \left(1 + \eta^2\epsilon^2 \beta^2 + 2\eta^2 \beta^2\epsilon -2\eta \frac{\beta \gamma}{\beta + \gamma}  + 2\eta \beta \epsilon\right)\norm{\theta - \theta'}^2 \\
	 & - \left(\frac{2\eta}{\beta + \gamma}  - \eta^2 \right) \left\|\lgrad{\theta'}{\theta} - \lgrad{\theta'}{\theta'}\right\|_2^2.
\end{align*}
If we take the step size $\eta$ to be small enough, namely $\eta\leq \frac{2}{\beta + \gamma}$, we get
\[
\left\|\gd(\theta) - \gd({\theta'})\right\|_2^2 \leq  \left(1 + \eta^2\epsilon^2 \beta^2 + 2\eta^2 \beta^2\epsilon -2\eta \frac{\beta \gamma}{\beta + \gamma}  + 2\eta \beta \epsilon\right) \norm{\theta - \theta'}^2.
\]
To ensure a contraction, we need $2\eta \frac{\beta\gamma}{\beta + \gamma} - \eta^2 \epsilon^2\beta^2 - 2\eta^2 \beta^2 \epsilon - 2\eta \beta \epsilon>0$. Canceling out $\eta\beta$, and assuming $\epsilon \leq 1$, it suffices to have $\frac{2\gamma}{\beta + \gamma} - 3\eta\epsilon\beta - 2\epsilon > 0$.
Therefore, if $\epsilon< \frac{\gamma}{(\beta + \gamma)(1 + 1.5 \eta\beta)}\leq 1$, the map $\gd$ is contractive. In particular, we have
\begin{align*}
\left\|\gd(\theta) - \gd({\theta'})\right\|_2 &\leq  \sqrt{\left(1 - \eta\left(2 \frac{\beta \gamma}{\beta + \gamma}  - \epsilon (3 \eta \beta^2 +  2\beta)\right)\right)} \norm{\theta - \theta'}\\
&\leq \left(1 - \eta\left( \frac{\beta \gamma}{\beta + \gamma}  - \epsilon (1.5\eta\beta^2 + \beta)\right)\right) \norm{\theta - \theta'},
\end{align*}
where we use the fact that $\sqrt{1 - x} \leq 1 - \frac{x}{2}$ for $x\in[0,1]$. This completes the proof of part (a).

Since we have shown $\gd$ is contractive, by the Banach fixed-point theorem we know that there exists a unique fixed point of $\gd$. That is, there exists $\thetaPS$ such that $\lgrad{\thetaPS}{\thetaPS}=0$. By convexity of the loss function, this means that $\thetaPS$ is the optimum of $\ploss{\thetaPS}{\theta}$ over $\theta$, which in turn implies that $\thetaPS$ is performatively stable. Recursively applying the result of part (a) we get the rate of convergence of RRM to $\thetaPS$:
\begin{align*}
\|\theta_t - \thetaPS\|_2 &\leq \left(1 - \eta\left( \frac{\beta \gamma}{\beta + \gamma}  - \epsilon (1.5\eta\beta^2 + \beta)\right)\right)^{t} \|\theta_0 - \thetaPS\|_2\\
&\leq \exp\left(- t\eta\left( \frac{\beta \gamma}{\beta + \gamma}  - \epsilon (1.5\eta\beta^2 + \beta)\right)\right) \|\theta_0 - \thetaPS\|_2,
\end{align*}
where in the last step we use the fact that $1-x\leq e^{-x}$. Setting this expression to be at most $\delta$ and solving for $t$ completes the proof.
\subsection{Proof of Theorem \ref{theorem:finite_samples}}

\paragraph{Proof of (a):}

We introduce the main proof idea and then present the full argument. The proof proceeds by case analysis. First, we show that if $\|\theta_t - \thetaPS\|_2 > \delta$, performing ERM ensures that with high probability $\|\theta_{t+1} - \thetaPS\|_2 \leq 2\epsilon \frac{\beta}{\gamma}\|\theta_{t} - \thetaPS\|_2$. Using our assumption that $\epsilon < \frac{\gamma}{2\beta}$, this implies that the iterate $\theta_{t+1}$ contracts toward $\thetaPS$.

On the other hand, if $\|\theta_t-\thetaPS\|_2\leq \delta$, we show that while ERM might not contract, it cannot push $\theta_{t+1}$ too far from $\thetaPS$ either. In particular, $\theta_{t+1}$ must be in a $\frac{\epsilon\beta}{2\gamma}\delta$-ball around $\thetaPS$. The proof then concludes by arguing that $\theta_t$ for $t\geq \frac{\log(\|\theta_0 - \thetaPS\|_2/\delta)}{\log(\gamma/2\epsilon\beta)}$ must enter a ball of radius $\delta$ around $\thetaPS$. Once this event occurs, no future iterate can exit the $\frac{\epsilon\beta}{2\gamma}\delta$-ball around $\thetaPS$.

\underline{Case 1:}\quad $\|\theta_t - \thetaPS\|_2> \delta$. If the current iterate is outside the ball, we show that with high probability the next iterate contracts towards a performatively stable point. In particular, 
\[
\norm{\theta_{t+1}- \thetaPS} \leq \frac{2\epsilon \beta}{\gamma}\norm{\theta_t - \thetaPS}.
\]
To prove this claim, we begin by showing that 
\begin{equation}
\label{eq:wasserstein_empirical}
W_1(\D^{n_t}(\theta_t),\D(\thetaPS)) \leq 2\epsilon \|\theta_t - \thetaPS\|_2, \text{ with probability } 1 - \frac{6p}{\pi^2t^2}.
\end{equation}
Since the $W_1$-distance is a metric on the space of distributions, we can apply the triangle inequality to get
$$W_1(\D^{n_t}(\theta_t),\D(\thetaPS))\leq W_1(\D^{n_t}(\theta_t), \D(\theta_t)) + W_1(\D(\theta_t), \D(\thetaPS)).$$
The second term is bounded deterministically by $\epsilon \norm{\theta_t - \thetaPS}$ due to $\epsilon$-sensitivity. By Theorem 2 of Fournier \& Guillin, 2015 \cite{Fournier2015}, for $n_t \geq \frac{1}{c_2(\epsilon \delta)^m}\log\left(\frac{t^2\pi^2 c_1}{6p}\right)$, the probability that the first term is greater than $\epsilon \delta$ is less that $\frac{6 p}{t^2 \pi^2}$. Here, the positive constants $c_1, c_2$ depend on $\alpha,\mu, \xi_{\alpha,\mu}$ and $m$. Therefore,  
$$W_1(\D^{n_t}(\theta_t),\D(\thetaPS)) \leq \epsilon \delta + \epsilon \norm{\theta_t - \thetaPS} \leq 2\epsilon \norm{\theta_t - \thetaPS}, \text{ with probability } 1 - \frac{6p}{\pi^2t^2}.$$
Using this, we can now prove that the iterates contract. Following the first steps of the proof of Theorem \ref{theorem:exact_min_strongly_convex}, we have that 
\begin{align}
	\begin{split}
	\label{main_eq_exact_minimization_samples} 
		&\left(G^{n_t}(\theta_t) - G(\thetaPS)\right)^\top\Big(\elgrad{\D^{n_t}(\theta_t)}{G^{n_t}(\theta_t)} - \elgrad{\D(\thetaPS)}{G^{n_t}(\theta_t)}\Big)\\
		+  &\left(G^{n_t}(\theta_t) - G(\thetaPS)\right)^\top\Big(\elgrad{
		\D(\thetaPS)}{G^{n_t}(\theta_t)} - \elgrad{\D(\thetaPS)}{G(\thetaPS)}\Big) \leq  0.
	\end{split}
\end{align}
Like in the proof of Theorem \ref{theorem:exact_min_strongly_convex}, the term $(G^{n_t}(\theta_t) - G(\thetaPS))^\top\elgrad{\D^{n_t}(\theta_t)}{G^{n_t}(\theta_t)}$ is $\norm{G^{n_t}(\theta_t) - G(\thetaPS)} \cdot \beta$ Lipschitz in $Z$. Using equation (\ref{eq:wasserstein_empirical}), with probability $1 - \frac{6p}{\pi^2t^2}$ we can bound the first term by 
\begin{align*}
&\left(G^{n_t}(\theta_t) - G(\thetaPS)\right)^\top\Big(\elgrad{\D^{n_t}(\theta_t)}{G^{n_t}(\theta_t)} - \elgrad{\D(\thetaPS)}{G^{n_t}(\theta_t)}\Big)\\
\geq~ &-2\epsilon \beta \left\|G^{n_t}(\theta_t) - G(\thetaPS)\right\|_2 \|\theta_t - \thetaPS\|_2.
\end{align*}
And by strong convexity, 
\begin{align*}
&\left(G^{n_t}(\theta_t) - G(\thetaPS)\right)^\top\Big(\elgrad{
		\D(\thetaPS)}{G^{n_t}(\theta_t)} - \elgrad{\D(\thetaPS)}{G(\thetaPS)}\Big)\\
		 \geq~ &\gamma \norm{G^{n_t}(\theta_t) - G(\thetaPS)}^2.
\end{align*}
Plugging back into equation (\ref{main_eq_exact_minimization_samples}), we conclude that with high probability
\[
\norm{\theta_{t+1}- \thetaPS} \leq \frac{2\epsilon \beta}{\gamma}\norm{\theta_t - \thetaPS}.
\]
Applying a union bound, we conclude that the iterates contract at every iteration where $\|\theta_t - \thetaPS\|_2> \delta$ with probability at least $1 - \sum_{t=1}^\infty \frac{6p}{\pi^2t^2} = 1 - p$. Therefore, for  $t \geq \left(1 - \frac{2\epsilon \beta}{\gamma}\right)^{-1}\log\left(\frac{\|\theta_0 - \thetaPS\|_2}{\delta}\right)$ steps we have 
$$\norm{\theta_t - \thetaPS} \leq \left(\frac{2\epsilon\beta}{\gamma}\right)^t\norm{\theta_0 - \thetaPS} \leq  \left(\frac{2\epsilon\beta}{\gamma}\right)^t \|\theta_0 - \thetaPS\|_2 \leq \exp\left(-t \left(1 - \frac{2\epsilon\beta}{\gamma}\right)\right) \|\theta_0 - \thetaPS\|_2 \leq \delta,$$
where we use $1 - x \leq e^{-x}$. This implies that $\theta_t$ eventually contracts to a ball of radius $\delta$ around $\thetaPS$.

\underline{Case 2:}\quad$\|\theta_t - \thetaPS\|_2\leq \delta$. We show that the RERM iterates can leave a ball of radius $\delta$ around $\thetaPS$ only with negligible probability. We begin by applying the triangle inequality just as we did in the previous case,
\[
W_1(\D^{n_t}(\theta_t),\D(\thetaPS))\leq W_1(\D^{n_t}(\theta_t), \D(\theta_t)) + W_1(\D(\theta_t), \D(\thetaPS)) \leq W_1(\D^{n_t}(\theta_t), \D(\theta_t)) + \epsilon \delta.
\]
For our choice of $n_t$, with probability at least $1 - \frac{6 p}{\pi^2t^2}$ this quantity is upper bounded by
$$W_1(\D^{n_t}(\theta_t), \D(\thetaPS)) \leq 2\epsilon \delta.$$
With this information, we can now apply the exact same steps as in the previous case, but now using the fact that $W_1(\D^{n_t}(\theta_t), \D(\thetaPS)) \leq 2\epsilon \delta$ instead of $W_1(\D^{n_t}(\theta_t), \D(\thetaPS)) \leq 2\epsilon \norm{\theta_t - \thetaPS}$, to conclude that with probability at least $1 - \frac{6p}{\pi^2t^2}$
\[
\norm{\theta_{t+1} - \thetaPS} \leq 2\epsilon \frac{\beta}{\gamma} \delta \leq \delta.
\]
As before, a union bound argument proves that the entire analysis holds with probability $1-p$.

\paragraph{Proof of (b):}

The only difference between part (b) in relation to part (a) is the fact that one needs to invoke the steps of Theorem \ref{theorem:contraction} rather than Theorem \ref{theorem:exact_min_strongly_convex}.

\subsection{Proof of Proposition \ref{propo:existence_stable_points}}

\label{app:propo_existence}

We begin by defining the set-valued function, $g(\theta) = \argmin_{\theta' \in \Theta} \map(\theta, \theta')$. Observe that fixed points of this function correspond to models which are performatively stable. The proof thereby follows from showing that the function $g(\cdot)$ has a fixed point.

 Since the loss is jointly continuous and the set $\Theta$ is compact, we can apply Berge's Maximum Theorem \cite{berge_maximum, aliprantis} to conclude that the function $g(\cdot)$ is upper hemicontinuous with compact and non-empty values. Furthermore,  by convexity of the loss, it follows that in addition to being compact and non-empty, $g(\theta)$ is a convex set for every $\theta \in \Theta$. Therefore, the conditions of Kakutani's Theorem \cite{kakutani} (also see Ch 17. in \cite{aliprantis}) hold and we can conclude that $g(\cdot)$ has a fixed point. Hence, a performatively stable model exists.

\subsection{Proof of Proposition \ref{prop:propoconcave}}

We make a slight modification to Example \ref{example:biasedcointoss} to prove the proposition. As in the example, $\D(\theta)$ is given as follows: $X$ is a single feature supported on $\{\pm 1\}$ and $Y~|~X  \sim \text{Bernoulli}(\frac 1 2 + \mu X + \epsilon \theta X)$, where $\Theta = [0,1]$. We let $\epsilon\geq \frac{1}{2}$, and constrain $\mu$ to satisfy $\abs{\mu + \epsilon} \leq \frac{1}{2}$. We assume that outcomes are predicted according to the model $f_\theta(x) = \theta x + \frac 1 2$ and that performance is measured via the squared loss, $\ell(z;\theta) = (y - f_\theta(x))^2$. This loss has condition number $\frac{\beta}{\gamma}=1$.

A direct calculation demonstrates that the performative risk is a quadratic in $\theta$:
\[
\PR(\theta) = \frac{1}{4} - 2\theta \mu + (1-2\epsilon)\theta^2.
\]
Therefore, if $\epsilon \in \left[\frac{1}{2},1\right)$, the performative risk is a concave function of $\theta$, even though $\epsilon< \frac{\gamma}{\beta}$.

\subsection{Proof of Theorem \ref{theorem:closeness}}

By definition of performative optimality and performative stability we have that:
\[
\map(\thetaPO,\thetaPO) \leq \map(\thetaPS, \thetaPS) \leq \map(\thetaPS, \thetaPO).
\]
We claim that $\map(\thetaPS, \thetaPO) - \map(\thetaPS,\thetaPS) \geq \frac{\gamma}{2}\|\thetaPO-\thetaPS\|_2^2$. By definition of $\map$, we can write
\[
\map(\thetaPS, \thetaPO) - \map(\thetaPS,\thetaPS) = \E_{Z\sim \D(\thetaPS)} \big[ \ell(Z; \thetaPO) - \ell(Z; \thetaPS) \big].
\]
Since $\ell(z; \thetaPO) \geq \ell(z; \thetaPS) + \nabla_\theta \ell(z; \thetaPS)^\top (\thetaPO - \thetaPS) + \frac{\gamma}{2}\|\thetaPO-\thetaPS\|_2^2$ for all $z$, we have that 
\begin{equation}
\label{eq:strong_convex_ineq}
\E_{Z\sim \D(\thetaPS)} \big[ \ell(Z; \thetaPO) - \ell(Z; \thetaPS) \big] \geq \E_{Z\sim \D(\thetaPS)} \big[ \nabla_\theta \ell(Z; \thetaPS)^\top (\thetaPO - \thetaPS) \big] + \frac{\gamma}{2}\|\thetaPO - \thetaPS\|_2^2.
\end{equation}
Now, by Lemma \ref{lemma:first_order_opt_condition}, $\E_{Z\sim \D(\thetaPS)} \big[ \nabla_\theta \ell(Z; \thetaPS)^\top(\thetaPO - \thetaPS) \big] \geq 0$, so we get that equation~(\ref{eq:strong_convex_ineq}) implies 
\[
\E_{Z\sim \D(\thetaPS)} \big[ \ell(Z; \thetaPO) - \ell(Z;\thetaPS) \big] \geq  \frac{\gamma}{2}\|\thetaPO - \thetaPS\|_2^2.
\]
Since the population distributions are $\epsilon$-sensitive and the loss is $L_z$-Lipschitz in $z$, we have that $\map(\thetaPS, \thetaPO) - \map(\thetaPO,\thetaPO) \leq L_z\epsilon\|\thetaPO -\thetaPS\|_2$. If $\epsilon < \frac{\gamma\|\thetaPO-\thetaPS\|_2}{2L_z}$ then we have that $L_z\epsilon\|\thetaPO -\thetaPS\|_2 < \frac{\gamma}{2}\|\thetaPO-\thetaPS\|_2^2$ which is a contradiction since it must hold that 
$$\map(\thetaPS, \thetaPO) - \map(\thetaPO,\thetaPO) \geq \map(\thetaPS, \thetaPO) - \map(\thetaPS,\thetaPS).$$

\subsection{Proof of Corollary \ref{corollary:strategic_classification}}

By Theorem \ref{theorem:exact_min_strongly_convex} we know that repeated risk minimization converges at a linear rate to a performatively stable point $\thetaPS$. Furthermore, by Theorem \ref{theorem:closeness}, this performatively stable point is close in domain to the institution's Stackelberg equilibrium classifier $\thetaSE$,
\[
\norm{\thetaSE - \thetaPS} \leq \frac{2L_z\epsilon}{\gamma}. 
\]
We can then use the fact that the loss is Lipschitz to show that this performatively stable classifier is close in objective value to the Stackelberg equilibrium:
\begin{align*}
\PR(\thetaPS) - \PR(\thetaSE) &\leq \big| \PR(\thetaPS) -  \map(\thetaPS, \thetaSE) \big| + \big|\map(\thetaPS, \thetaSE) -  \PR(\thetaSE) \big| \\
&\leq L_\theta \norm{\thetaSE - \thetaPS} + L_z\epsilon \norm{\thetaSE - \thetaPS}\\
&\leq \frac{2L_z \epsilon (L_\theta + L_z\epsilon)}{\gamma}
\end{align*}
Here, we have used the Kantorovich-Rubinstein Lemma (\ref{lemma:duality}) to bound the second term.

\section{Approximately minimizing performative risk via regularization}
\label{sec:regularization}

Recall that in Proposition \ref{prop:tightness} we have shown that RRM might not converge at all if the objective is smooth and convex, but not strongly convex. In this section, we show how adding a small amount of quadratic regularization to the objective guarantees that RRM will converge to a stable point which approximately minimizes the performative risk on the original loss. 

To do so, we additionally require that the space of model parameters $\Theta$ be bounded with diameter $D =\sup_{\theta,\theta' \in \Theta} \|\theta - \theta'\|_2$. We can assume without loss of generality that $D=1$.

\begin{proposition}
\label{prop:regularization}
Suppose that the loss $\ell(z;\theta)$ is $L_z$-Lipschitz iz $z$ and $L_\theta$-Lipschitz in $\theta$, $\beta$-jointly smooth  \eqref{ass:a2} and convex (but not necessarily strongly convex). Furthermore, suppose that distribution map $\D(\cdot)$ is $\epsilon$-sensitive with $\epsilon < 1$, and that the set $\Theta$ is bounded with diameter 1. Then, there exists a choice of $\alpha$, such that running RRM with loss $\ell^{\reg}(z;\theta) \defeq \ell(z;\theta) + \frac{\alpha}{2}\|\theta - \theta_0\|^2_2$ converges to a performatively stable point $\theta_{\mathrm{PS}}^{\reg}$ which satisfies the following
$$\PR(\theta_{\mathrm{PS}}^{\reg})\leq \min_\theta \PR(\theta) + O\left(\frac{\sqrt{\epsilon}}{1-\epsilon}\right).$$
\end{proposition}

We note that in the case where $\epsilon = 0$, the limit point $\theta_{\mathrm{PS}}^{\reg}$ of regularized repeated risk minimization is also performatively optimal.

\begin{proof}

	First, we observe that the regularized loss function $\ell^{\reg}(z;\theta)$ is $\alpha$-strongly convex and $\alpha+\beta$-jointly smooth. Since $\epsilon < 1$, we can then choose an $\alpha$ such that $\epsilon<\frac{\alpha}{\alpha + \beta}$. In particular, we choose $\alpha =\sqrt{\epsilon}\beta / (1-\epsilon)$.
	
	From our choice of $\alpha$, we have that $\epsilon$ is smaller than the inverse condition number. Hence, by Theorem \ref{theorem:exact_min_strongly_convex} repeated risk minimization converges at a linear rate to a performatively stable solution $\theta_{\mathrm{PS}}^{reg}$ of the regularized objective. 
	
	To finish the proof, we show that the objective value at the $\theta_{\mathrm{PS}}^{reg}$ is close to the objective value at the performative optima of the  original objective $\theta_{\mathrm{PO}}$. We do so by bounding their difference using the triangle inequality:
\begin{align*}
\pregloss{\theta_{\mathrm{PS}}^{\reg}}{\theta_{\mathrm{PS}}^{\reg}} - \ploss{\theta_{\mathrm{PO}}}{\theta_{\mathrm{PO}}} &= \pregloss{\theta_{\mathrm{PS}}^{\reg}}{\theta_{\mathrm{PS}}^{\reg}} - \pregloss{\theta_{\mathrm{PO}}^{\reg}}{\theta_{\mathrm{PO}}^{\reg}}\\
&+ \pregloss{\theta_{\mathrm{PO}}^{\reg}}{\theta_{\mathrm{PO}}^{\reg}} - \ploss{\theta_{\mathrm{PO}}}{\theta_{\mathrm{PO}}}
\end{align*}
We can bound the first difference via Lipschitzness:
	\begin{align*}
		\pregloss{\theta_{\mathrm{PS}}^{\reg}}{\theta_{\mathrm{PS}}^{\reg}} - \pregloss{\theta_{\mathrm{PO}}^{\reg}}{\theta_{\mathrm{PO}}^{\reg}} &= \pregloss{\theta_{\mathrm{PS}}^{\reg}}{\theta_{\mathrm{PS}}^{\reg}} - \pregloss{\theta_{\mathrm{PS}}^{\reg}}{\theta_{\mathrm{PO}}^{\reg}}\\
		&+ \pregloss{\theta_{\mathrm{PS}}^{\reg}}{\theta_{\mathrm{PO}}^{\reg}} - \pregloss{\theta_{\mathrm{PO}}^{\reg}}{\theta_{\mathrm{PO}}^{\reg}}\\
		&\leq (L_\theta+\alpha \sup_{\theta,\theta' \in\Theta} \|\theta - \theta'\|_2)\|\theta_{\mathrm{PS}}^{\reg} - \theta_{\mathrm{PO}}^{\reg}\|_2\\
		&+ \epsilon L_z\|\theta_{\mathrm{PS}}^{\reg} - \theta_{\mathrm{PO}}^{\reg}\|_2\\
		&= (L_\theta + \alpha + \epsilon L_z)\|\theta_{\mathrm{PS}}^{\reg} - \theta_{\mathrm{PO}}^{\reg}\|_2\\
		&\leq \frac{2(L_\theta + \alpha + \epsilon L_z) L_z \epsilon}{\alpha}.
	\end{align*}
In the last two lines, we have applied the fact that $D = \sup_{\theta,\theta'\in\Theta} \|\theta - \theta'\|_2 = 1$ as well as Theorem \ref{theorem:closeness}. For the second difference, by definition of performative optimality we have that,
	
	$$\pregloss{\theta_{\mathrm{PO}}^{\reg}}{\theta_{\mathrm{PO}}^{\reg}} \leq \pregloss{\theta_{\mathrm{PO}}}{\theta_{\mathrm{PO}}}  \leq  \ploss{\theta_{\mathrm{PO}}}{\theta_{\mathrm{PO}}} + \frac{\alpha}{2}.$$
Where we have again used the fact that $D=1$ for the last inequality. Combining these two together, we can bound the total difference:
	$$\pregloss{\theta_{\mathrm{PS}}^{\reg}}{\theta_{\mathrm{PS}}^{\reg}} - \ploss{\theta_{\mathrm{PO}}}{\theta_{\mathrm{PO}}} \leq \frac{2(L_\theta + \alpha + \epsilon L_z) L_z \epsilon}{\alpha} + \frac{\alpha}{2}.$$
Plugging in $\alpha = \frac{\sqrt{\epsilon}\beta}{1-\epsilon}$ completes the proof.
\end{proof}

\end{document}